\documentclass[a4paper,fleqn]{cas-dc}

\usepackage[authoryear]{natbib}

\usepackage{amsmath}
\usepackage{amssymb}
\usepackage{amsthm}
\newtheorem{theorem}{Theorem}

\newtheorem{proposition}{Proposition}

\usepackage{url}
\usepackage{xcolor}
\usepackage{hyperref}

\usepackage{booktabs}
\usepackage{multirow}
\usepackage{subcaption}
\usepackage{algorithm}
\usepackage{algpseudocode}
\usepackage{enumitem}

\def\tsc#1{\csdef{#1}{\textsc{\lowercase{#1}}\xspace}}
\tsc{WGM}
\tsc{QE}
\tsc{EP}
\tsc{PMS}
\tsc{BEC}
\tsc{DE}

\begin{document}
\let\WriteBookmarks\relax
\def\floatpagepagefraction{1}
\def\textpagefraction{.001}
\shorttitle{AdamFlow: Adam-based Wasserstein Gradient Flows for Surface Registration}
\shortauthors{Q. Ma et al.}

\title [mode = title]{AdamFlow: Adam-based Wasserstein Gradient Flows for Surface Registration in Medical Imaging}              

\author[1,2]{Qiang Ma}[orcid=0000-0003-0791-1731]
\cormark[1]
\author[2,3]{Qingjie Meng}[orcid=0000-0001-8728-4007]
\author[4]{Xin Hu}
\author[1,2]{Yicheng Wu}[orcid=0000-0002-7669-9167]
\cormark[1]
\author[1,2,5]{Wenjia Bai}[orcid=0000-0003-2943-7698]
\cormark[1]

\affiliation[1]{organization={Department of Brain Sciences, Imperial College London}, city={London}, country={UK}}
\affiliation[2]{organization={Department of Computing, Imperial College London}, city={London}, country={UK}}
\affiliation[3]{organization={School of Computer Science, University of Birmingham}, city={Birmingham}, country={UK}}
\affiliation[4]{organization={Department of Computer Science, Columbia University}, city={New York}, state={NY}, country={USA}}
\affiliation[5]{organization={Data Science Institute, Imperial College London}, city={London}, country={UK}}

\cortext[cor1]{Corresponding authors}

\begin{abstract}
Surface registration plays an important role for anatomical shape analysis in medical imaging. Existing surface registration methods often face a trade-off between efficiency and robustness. Local point matching methods are computationally efficient, but vulnerable to noise and initialisation. Methods designed for global point set alignment tend to incur a high computational cost. To address the challenge, here we present a fast surface registration method, which formulates surface meshes as probability measures and surface registration as a distributional optimisation problem. The discrepancy between two meshes is measured using an efficient sliced Wasserstein distance with log-linear computational complexity. We propose a novel optimisation method, AdamFlow, which generalises the well-known Adam optimisation method from the Euclidean space to the probability space for minimising the sliced Wasserstein distance. We theoretically analyse the asymptotic convergence of AdamFlow and empirically demonstrate its superior performance in both affine and non-rigid surface registration across various anatomical structures.
\end{abstract}




\maketitle

\section{Introduction}
Surface registration is a fundamental task in medical image analysis and anatomical shape modelling, which aims at establishing point-wise spatial correspondence between anatomical structures represented as 3D surface meshes \citep{audette2000registration,tam2012registration,ma2017ot}. In contrast to image registration, surface registration operates directly on geometric representations, facilitating the alignment of complex surface features. Accurate and robust shape registration is crucial for quantifying inter-subject anatomical variability and intra-subject longitudinal changes, thereby providing the foundation for population-based shape analysis \citep{patenaude2011bayesian,robinson2014msm,bai2015atlas,burns2024genetic}, longitudinal studies \citep{thompson2004change,gerig2016longitudinal,garcia2018dynamic}, and statistical shape modelling \citep{frangi2003ssm,heimann2009ssm,ludke2022flowssm,bastian2023s3m}.

Over the past decades, a wide range of point cloud and surface registration strategies have been proposed \citep{besl1992icp,audette2000registration,chui2003tps,myronenko2006cpd,yeo2009spherical,tam2012registration,lombaert2013diffeomorphic,robinson2014msm}. Early approaches performed rigid or affine local point matching using the variants of the iterative closest point (ICP) algorithm \citep{besl1992icp,zhang1994iterative,rusinkiewicz2001icp}, which align two point sets by iteratively assigning nearest neighbour correspondences and estimating the transformation, achieving a computational complexity of $\mathcal{O}(N\log N)$ for $N$ points. Despite its efficiency and simplicity, such a local correspondence strategy is sensitive to initialisation of the transformation, noise, and outliers, and is therefore prone to being trapped in local minima during the optimisation procedure. 

To model rigid and non-rigid transformations of the global 3D structure, advanced shape registration methods have been introduced, including spline-based deformation models such as the thin plate spline \citep{bookstein2002tps,chui2003tps}, probabilistic frameworks such as the coherent point drift algorithm \citep{myronenko2006cpd,myronenko2010cpd}, and diffeomorphic registration that enforces topology-preserving transformations \citep{yeo2009spherical,charon2013varifold,ma2017ot}. However, most non-rigid registration approaches are computationally intensive with $\mathcal{O}(N^3)$ complexity due to solving a kernel-based linear system, which severely limits their application to high-resolution meshes containing tens of thousands of vertices. Although deep learning-based shape registration exhibits notable efficiency and scalability at inference time \citep{wang2019deep,lu2019deepvcp,aoki2019pointnetlk,sun2022ndf,le2024correspond}, they often face challenges in generalising across datasets or anatomical structures that are different from the training data.

To capture the global discrepancy between two shapes without relying on local correspondences, recent approaches model point clouds or surface meshes as probability measures and formulate shape registration as an optimisation problem to find an optimal transport distance, \emph{e.g.}, the Wasserstein distance, between two probability measures \citep{solomon2015swd,peyre2019ot}. However, the scalability of the vanilla Wasserstein distance is limited due to its  complexity of $\mathcal{O}(N^3\log N)$ \citep{villani2008ot,chewi2024sot}. As an efficient alternative approximation, the sliced Wasserstein distance and its variants were proposed \citep{bonneel2015swd,kolouri2019swd,nguyen2023swd,nguyen2024swd}, which reduce the complexity to $\mathcal{O}(N\log N)$. This is achieved by projecting high-dimensional distributions onto one-dimensional subspaces, where the optimal transport problem admits a closed-form solution, and then averaging the transport costs over multiple projection directions. Sliced Wasserstein distance has been widely adopted in point cloud registration tasks \citep{lai2017swd,nguyen2021swd,nguyen2023swd,nguyen2024swd}, while its application to surface registration remains relatively underexplored \citep{le2023diffeomorphic,nguyen2024swd}, which requires not only geometric alignment of point sets but also the preservation of mesh connectivity.

To minimise the sliced Wasserstein distance, existing approaches typically employ the Wasserstein gradient flow, which leverages a continuity partial differential equation to evolve a point cloud or a surface mesh toward a target probability measure \citep{jordan1998jko,ambrosio2005wgf,santambrogio2017wgf,chewi2024sot}. Analogous to gradient descent in the Euclidean space, Wasserstein gradient flow minimises an objective functional following the direction of the Wasserstein gradient in the space of probability measures \citep{ambrosio2005wgf,santambrogio2017wgf}. To accelerate the convergence of the Wasserstein gradient flow, a recent study \citep{chen2025accelerating} generalised classical momentum-based optimisation, such as the heavy ball or Nesterov method \citep{nesterov1983method}, from the Euclidean space to the probability space. It demonstrated the theoretical and empirical effectiveness in convex distributional optimisation. However, the convergence is not guaranteed for non-convex shape registration problems.

For non-convex optimisation problems in the Euclidean space, the Adam optimiser and its variants have achieved remarkable empirical performance \citep{kingma2015adam,dozat2016nadam,reddi2018amsgrad,loshchilov2019adamw,liu2020radam,barakat2021convergence,xiao2024adam}, particularly for high-dimensional optimisation such as training of large neural networks. The Adam optimiser adjusts the learning rate adaptively by estimating the first and second moments of the gradient, which enables it to navigate complex loss landscapes \citep{kingma2015adam}. To the best of our knowledge, few studies have incorporated the Adam method into the Wasserstein gradient flow for non-convex distributional optimisation, despite its strong potential for efficient surface registration.

In this work, we formulate surface registration in medical imaging as a distributional optimisation problem, by modelling surface meshes as probability measures. The discrepancy between two meshes is measured by the sliced Wasserstein distance. We develop algorithms for both affine and non-rigid surface registration. For affine surface registration, we demonstrate that the sliced Wasserstein distance achieves consistently better registration performance and lower computational cost compared to the traditional ICP algorithm. For non-rigid surface registration, we propose a hybrid optimisation strategy, which first optimises the sliced Wasserstein distance to capture global shape alignment and then minimises a Chamfer distance to ensure accurate local geometric matching, enabling coarse-to-fine registration.

To accelerate the optimisation of the sliced Wasserstein distance, we propose a novel method AdamFlow, a variant of Wasserstein gradient flow that adaptively estimates the first and second moments of the Wasserstein gradients, extending the Adam optimiser from the Euclidean space to the space of probability measures. We theoretically establish the asymptotic convergence of AdamFlow, showing that its solution converges to the set of critical points of the objective functional without the requirement of convexity. Using a particle-based numerical implementation, we empirically demonstrate that AdamFlow achieves faster convergence and superior performance on both affine and non-rigid mesh registration tasks for various anatomical structures (the liver, the pancreas, and the left ventricle of the heart), compared to existing momentum-based distributional optimisation methods. Our codes are publicly available at \href{https://github.com/m-qiang/AdamFlow}{https://github.com/m-qiang/AdamFlow}. The main contributions of this work are listed as follows:
\begin{enumerate}
\item  We formulate surface registration as a distributional optimisation problem and introduce the sliced Wasserstein distance to improve both affine and non-rigid surface registration.
\item  We propose AdamFlow, which generalises the Adam optimisation method to the space of probability measures, and theoretically prove its asymptotic convergence for non-convex distributional optimisation.
\item We empirically demonstrate that AdamFlow leads to improved registration accuracy and faster convergence in surface registration tasks across multiple organs, compared to existing momentum-based methods.
\end{enumerate}

\section{Preliminaries}
In this work, we model a surface mesh as a probability measure and formulate surface registration as a distributional optimisation problem. Here, we introduce necessary background for distributional optimisation in the space of probability measures.

\paragraph{Distributional optimisation.}
Let $\mathcal{P}(\mathbb{R}^d)$ be the set of all probability measures $\mu$ supported on $\mathbb{R}^d$. An unconstrained \textit{distributional optimisation} problem is defined as
\begin{equation}\label{eq:d_opt}
\min_{\mu\in\mathcal{P}(\mathbb{R}^d)}F[\mu],
\end{equation}
where $F:\mathcal{P}(\mathbb{R}^d)\rightarrow\mathbb{R}$ denotes an objective functional of the probability measure $\mu$. The distributional optimisation problem aims to find an optimal probability measure $\mu_*$ such that the functional $F[\cdot]$ is minimised. Table~\ref{tab:wasserstein_grad} provides commonly used objective functionals in distributional optimisation \citep{chewi2024sot}.

\paragraph{Wasserstein distance.} In distributional optimisation, the discrepancy between two probability measures can be characterised by their \textit{Wasserstein distance}. Let $\mathcal{P}_p(\mathbb{R}^d)$ be the set of probability measures $\mu$ supported on $\mathbb{R}^d$ with finite $p$-th moment, \emph{i.e.,}
$\int_{\mathbb{R}^d}\|x\|^p\mathrm{d}\mu(x)<\infty$. For two probability measures $\mu,\nu\in\mathcal{P}_p(\mathbb{R}^d)$, their \textit{Wasserstein distance} is defined by
\begin{equation}\label{eq:wasserstein}
W_p(\mu,\nu)=\inf_{\gamma\in\Gamma(\mu,\nu)}\left(\int_{\mathbb{R}^d\times\mathbb{R}^d}\|x-y\|^p\mathrm{d} \gamma(x,y)\right)^{1/p},
\end{equation}
where $\Gamma(\mu,\nu)$ denotes the set of all couplings of $\mu$ and $\nu$. In practice, the computation of the Wasserstein distance can be expensive. For two discrete probability distributions $\mu$ and $\nu$ supported on $N$ points, the Wasserstein distance $W_p(\mu,\nu)$ is calculated by linear programming with $\mathcal{O}(N^3\log N)$ time complexity \citep{villani2008ot,chewi2024sot}.

\paragraph{Sliced Wasserstein distance.} As a computationally efficient approximation of the Wasserstein distance, the \textit{sliced Wasserstein distance} (SWD) uses \textit{Radon transform} to project probability measures onto one-dimensional subspaces \citep{bonneel2015swd}, where the Wasserstein distance can be calculated efficiently with a time complexity of $\mathcal{O}(N\log N)$. For probability measures $\mu,\nu\in\mathcal{P}_p(\mathbb{R}^d)$, the SWD between probability measures $\mu$ and $\nu$ is defined as
\begin{equation}\label{eq:sliced_wasserstein}
SW_p(\mu,\nu)=\left(\int_{\mathbb{S}^{d-1}}W_p^p(\pi_{\theta}\sharp\mu,\pi_{\theta}\sharp\nu)\mathrm{d}\theta\right)^{1/p},
\end{equation}
where $\mathbb{S}^{d-1}\subset\mathbb{R}^d$ denotes a unit sphere, $\theta\in\mathbb{S}^{d-1}$ denotes a unit vector with $\|\theta\|=1$, $\pi_{\theta}:\mathbb{R}^d\rightarrow\mathbb{R}$ denotes the projection function defined as $\pi_{\theta}(x)=\theta^{\top}x$, and $\pi_{\theta}\sharp\mu$ denotes an one-dimensional \textit{pushforward} measure of $\mu$, defined by $(\pi_{\theta}\sharp\mu)(x):=\mu(\pi_{\theta}^{-1}(x))$.
Numerically, the SWD can be approximated by Monte Carlo sampling with $L$ projections \citep{bonneel2015swd}:
\begin{equation}\label{eq:sliced_wasserstein_mc}
SW_p(\mu,\nu)\approx\left(\frac{1}{L}\sum_{l=1}^L W_p^p(\pi_{\theta_l}\sharp\mu,\pi_{\theta_l}\sharp\nu)\right)^{1/p},
\end{equation}
where the Wasserstein distance $W_p(\pi_{\theta_l}\sharp\mu,\pi_{\theta_l}\sharp\nu)$ between one-dimensional probability measures has a closed-form solution \citep{villani2008ot,chewi2024sot}.

\paragraph{Wasserstein gradient.} 
Unlike optimisation in the Euclidean space, distributional optimisation is performed in the \textit{Wasserstein space} $(\mathcal{P}_p(\mathbb{R}^d),W_p)$, a metric space over probability measures equipped with Wasserstein metric $W_p$. In this work, we only consider the 2-Wasserstein metric $W_2$, which is analogous to the $L^2$ norm in the Euclidean space. 

In the Euclidean space, gradient-based optimisation methods search for the minima of an objective function following the steepest descent, \emph{i.e.}, the direction of gradient. In the Wasserstein space, the concept of gradient can be extended to a functional of probability measures. According to the Otto calculus \citep{otto2001geometry}, the \textit{Wasserstein gradient} $\nabla_{W}F[\mu]:\mathbb{R}^d\rightarrow\mathbb{R}^d$ of a functional $F$ is defined as
\begin{equation}
    \nabla_{W}F[\mu]:=\nabla\delta F[\mu],
\end{equation}
where $\delta F[\mu]:\mathbb{R}^d\rightarrow\mathbb{R}$ denotes the first variation of the functional $F$ with respect to the probability measure $\mu$, and $\nabla$ denotes the gradient in the Euclidean space. Table~\ref{tab:wasserstein_grad} provides the Wasserstein gradients of common objective functionals in distributional optimisation problems \citep{ambrosio2005wgf,chewi2024sot}. For sliced 2-Wasserstein distance $\frac{1}{2}SW_2^2(\mu,\mu_*)$, the Wasserstein gradient with respect to $\mu$ is defined as \citep{cozzi2025long}
\begin{equation}\label{eq:sliced_wasserstein_grad}
\begin{split}
\frac{1}{2}\nabla_W SW_2^2(\mu,\mu_*)(x)
&=\int_{\mathbb{S}^{d-1}}\left(\theta^{\top}x-\mathcal{T}_{\theta}(\theta^{\top}x)\right)\theta\mathrm{d}\theta\\
&\approx\frac{1}{L}\sum_{l=1}^{L}\left(\theta_l^{\top}x-\mathcal{T}_{\theta_l}(\theta_l^{\top}x)\right)\theta_l,\\
\end{split}
\end{equation}
where $\mathcal{T}_{\theta}$ denotes the 1D optimal transport map from the projection $\pi_{\theta}\sharp\mu$ to $\pi_{\theta}\sharp\mu_*$.

\begin{table}[!h]
\centering
\caption{Commonly used objective functionals and their Wasserstein gradients. For 2-Wasserstein distance, $\mathcal{T}:\mathbb{R}^d\rightarrow\mathbb{R}^d$ denotes the optimal transport map from $\mu$ to $\mu_*$ such that the pushforward measure $\mathcal{T}\sharp\mu=\mu_*$.}
\begin{tabular}{ccc}
\toprule
Objective & $F[\mu]$ & $\nabla_W F[\mu]$ \\
\midrule
Potential energy &
$\int_{\mathbb{R}^d} V(x)\mathrm{d}\mu(x)$ &
$\nabla V(x)$ \\
Internal energy & $\int_{\mathbb{R}^d}U(\mu(x))\mathrm{d}x$ &
$\nabla U'(\mu(x))$ \\
Negative entropy & 
$\int_{\mathbb{R}^d}\mu(x)\log\mu(x)\mathrm{d}x$ &
$\nabla\log\mu(x)$ \\
KL divergence & $KL(\mu||\mu_*)$ & 
$\nabla\log\frac{\mu(x)}{\mu_*(x)}$ \\
2-Wasserstein distance &
$\frac{1}{2}W_2^2(\mu,\mu_*)$ &
$x-\mathcal{T}(x)$\\
\bottomrule
\end{tabular}
\label{tab:wasserstein_grad}
\end{table}

Based on the definition of the Wasserstein gradient, a probability measure $\mu_*\in\mathcal{P}_2(\mathbb{R}^d)$ is called a \textit{critical point} of the functional $F[\mu]$, if the Wasserstein gradient of $F$ satisfies $\nabla_W F[\mu_*](x)=0$ for almost everywhere (a.e.) $x\in\textrm{supp}(\mu_*)$, where $\textrm{supp}(\mu_*)$ denotes the support of the probability measure $\mu_*$.

\paragraph{Wasserstein gradient flow.} 
To perform distributional optimisation, the \textit{Wasserstein gradient flow} (WGF) is employed and defined by the following continuity equation \citep{jordan1998jko,ambrosio2005wgf,santambrogio2017wgf}:
\begin{equation}\label{eq:wgf}
\partial_t\mu_t=\nabla\cdot(\mu_t\nabla_W F[\mu_t]),
\end{equation}
where $t\in\mathbb{R}_+$ and $\mu_t:\mathbb{R}\rightarrow\mathcal{P}_2(\mathbb{R}^d)$ denotes a curve of probability measure that evolves with time $t$. Solving the partial differential equation (PDE) (\ref{eq:wgf}) enables us to find the critical points of the objective functional $F$. This is analogous to the gradient flow in the Euclidean space, \emph{i.e.}, the continuous form of the gradient descent.

\begin{figure*}[t]
\centering
\includegraphics[width=1.0\linewidth]{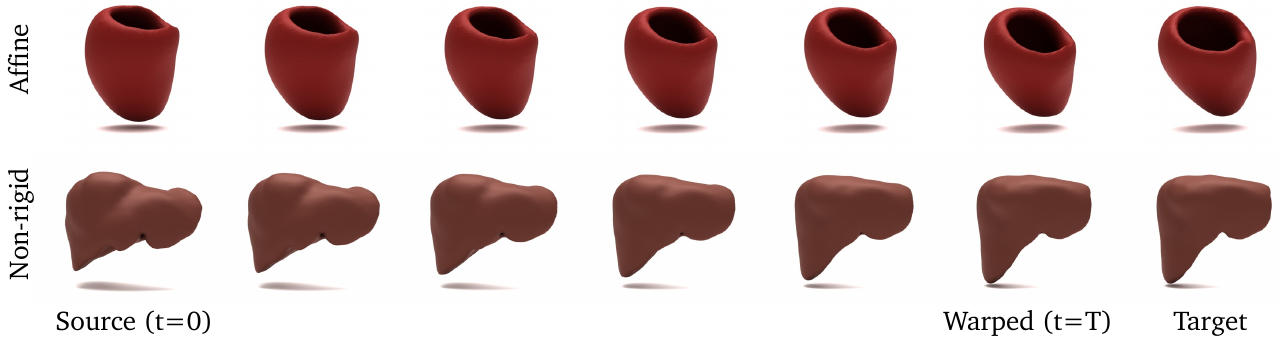}
\caption{Illustration of affine surface registration of the left ventricle of the heart and non-rigid surface registration of the liver. The surface registration is formulated as a distributional optimisation problem between a source surface ($t=0$) and a target surface. The registration is performed by integrating a Wasserstein gradient flow for $t\in[0,T]$ to optimise an objective functional.}
\label{fig:mesh_registration}
\end{figure*}

\section{Surface registration}
In this section, we model a surface mesh as a probability measure and formulate surface registration as distributional optimisation problems over the Wasserstein space. For both affine and non-rigid registration problems as illustrated in Figure~\ref{fig:mesh_registration}, we define the optimisation objective based on the SWD to efficiently measure the discrepancy between a warped source surface and a target surface.

\subsection{Surface mesh as a probability measure}
Let $\mathcal{M}=(\mathcal{V},\mathcal{E},\mathcal{F})$ be a triangular mesh with vertex set $\mathcal{V}\subset\mathbb{R}^3$, edge set $\mathcal{E}$, and face set $\mathcal{F}$. Such a discrete mesh representation can be modelled by a continuous probability measure $\mu^{\mathcal{M}}$ \citep{le2023diffeomorphic}, where each point on the mesh follows the distribution 
\begin{equation}\label{eq:mesh_distribution}
\begin{split}
\mu^{\mathcal{M}}(x)
&=\sum_{i=1}^{|\mathcal{F}|}\mu^{\mathcal{M}}(x|f^i)\mu^{\mathcal{M}}(f^i)\\
&=\sum_{i=1}^{|\mathcal{F}|}\Bigg(\frac{1}{\mathrm{Vol}(f^i)}\sum_{i=1}^{|\mathcal{F}|}\frac{\mathrm{Vol}(f^i)}{\sum_{j=1}^{|\mathcal{F}|}\mathrm{Vol}(f^j)}\delta_{f^i}\Bigg),
\end{split}
\end{equation}
where $f^i\in\mathcal{F}$ is the $i$-th face, $|\mathcal{F}|$ is the number of faces, $\mathrm{Vol}(f)$ is the area of the face $f^i$, and $\delta_{f^i}$ is a Dirac measure at the face $f^i$. Intuitively, to sample a point $x$ on the mesh $\mathcal{M}$, we first sample a face $f^i$ from the categorical measure $\mu^{\mathcal{M}}(f^i)$ according to the face areas, and then sample a point $x$ on the face $f^i$ with probability $1/\mathrm{Vol}(f^i)$. In addition to the face-based probability measure, a mesh $\mathcal{M}$ can also be approximated by an empirical measure of the vertices, \emph{i.e.},
\begin{equation}\label{eq:mesh_empirical}
\mu^{\mathcal{M}}\approx\hat{\mu}^{\mathcal{M}}:=\frac{1}{|\mathcal{V}|}\sum_{i=1}^{|\mathcal{V}|}\delta_{q^i},
\end{equation}
where $q^i\in\mathcal{V}$ is the $i$-th vertex, $|\mathcal{V}|$ is the number of vertices.

\subsection{Surface registration in the Wasserstein space}
Given a source mesh $\mathcal{M}_0=(\mathcal{V}_0,\mathcal{E}_0,\mathcal{F}_0)$ with probability measure $\mu^{\mathcal{M}_0}$ and a target mesh $\mathcal{M}_*=(\mathcal{V}_*,\mathcal{E}_*,\mathcal{F}_*)$ with probability measure $\mu^{\mathcal{M}_*}$, surface registration aims to find a transformation $\phi:\mathbb{R}^3\rightarrow\mathbb{R}^3$ to warp the source mesh to the target surface, such that the discrepancy $D(\phi\sharp\mu^{\mathcal{M}_0},\mu^{\mathcal{M}_*})$ between the two meshes is minimised. We can formulate surface registration as the following distributional optimisation problem in the Wasserstein space:
\begin{equation}\label{eq:registration}
\min_{\mu=\phi\sharp\mu^{\mathcal{M}_0}}F[\mu]:=D(\mu,\mu^{\mathcal{M}_*})+R[\mu],
\end{equation}
where $D:\mathcal{P}_2(\mathbb{R}^3)\times\mathcal{P}_2(\mathbb{R}^3)\rightarrow\mathbb{R}$ is the discrepancy functional, and $R:\mathcal{P}_2(\mathbb{R}^3)\rightarrow\mathbb{R}$ is the mesh regularisation term. In order to stabilise the optimisation procedure, we approximate the source mesh $\mu^{\mathcal{M}_0}$ by the empirical measure of its vertices $\hat{\mu}^{\mathcal{M}_0}$ defined in Eq.~(\ref{eq:mesh_empirical}).

We suppose $\phi_t$ varies with time $t$ during optimisation and define $\mu_t:=\phi_t\sharp\mu^{\mathcal{M}_0}$. Then, the mesh registration problem (\ref{eq:registration}) can be solved using the WGF (\ref{eq:wgf}) with Wasserstein gradient $\nabla_WF[\mu]$, analogous to iterative gradient descent in the Euclidean space. In practice, the surface registration problem (\ref{eq:registration}) is usually non-convex, as the objective functional $F$ is defined by a non-convex discrepancy between probability measures. Therefore, instead of seeking a global optimal transformation $\phi_*$, we aim to find a transformation $\phi_{\infty}$ such that the probability measure $\mu_{\infty}=\phi_{\infty}\sharp\mu^{\mathcal{M}_0}$ is a critical point of the objective functional $F[\mu]$ as $t\rightarrow\infty$. Next, we will discuss cases where the transformation $\phi$ is affine and non-rigid respectively.

\subsection{Affine surface registration}
For affine surface registration, we define $\phi(x):=Ax+b$ as an affine transformation with affine matrix $A\in\mathbb{R}^3\times\mathbb{R}^3$ and translation vector $b\in\mathbb{R}^3$. The transformed mesh is expressed as an empirical measure $\phi\sharp\mu^{\mathcal{M}_0}\approx\phi\sharp\hat{\mu}^{\mathcal{M}_0}=\frac{1}{|\mathcal{V}|}\sum_{i=1}^{|\mathcal{V}|}\delta_{Aq^i_0+b}$ for $q^i_0\in\mathcal{V}_0$. The affine matrix and translation can be optimised by minimising a discrepancy functional without the mesh regularisation term.

For the classic ICP affine registration algorithm \citep{besl1992icp}, the discrepancy is defined as
\begin{equation}\label{eq:icp}
D_{\mathrm{icp}}(\mu,\mu_*)=\frac{1}{2}\int_{\mathbb{R}^3}\inf_{x_*\in\mathrm{supp}(\mu_*)}\|x-x_*\|^2\mathrm{d}\mu(x),
\end{equation}
which measures the nearest distance from the transformed source mesh $\mu=\phi\sharp\mu^{\mathcal{M}_0}$ to the target mesh $\mu_*=\mu^{\mathcal{M}_*}$. However, the local correspondence strategy of the ICP algorithm may cause the optimisation to be trapped in local minima. Therefore, in this study, we leverage the sliced Wasserstein discrepancy 
$SW_2(\mu,\mu_*)$ to efficiently capture non-local mesh alignment for affine surface registration. The distributional optimisation problem (\ref{eq:registration}) for affine surface registration can be re-formulated as:
\begin{equation}\label{eq:affine_registration}
\min_{\mu=\phi\sharp\mu^{\mathcal{M}_0}}F[\mu]=\frac{1}{2}SW_2^2(\mu,\mu^{\mathcal{M}_*}).
\end{equation}
The Wasserstein gradient of $F[\mu]$ is provided by Eq.~(\ref{eq:sliced_wasserstein_grad}).

\subsection{Non-rigid surface registration}\label{sec:nonrigid}
For non-rigid surface registration, we define $\phi(q^i_0):=q^i_0+\Delta q^i$ as vertex-wise transformation, in which $\Delta q^i\in\mathbb{R}^3$ denotes the displacement of the $i$-th vertex $q^i_0\in\mathcal{V}_0$. The transformed source mesh is expressed as $\phi\sharp\mu^{\mathcal{M}_0}\approx\phi\sharp\hat{\mu}^{\mathcal{M}_0}=\frac{1}{|\mathcal{V}|}\sum_{i=1}^{|\mathcal{V}|}\delta_{q^i_0+\Delta q^i}$. To regularise non-rigid surface deformation, we introduce the following mesh Laplacian constraint:
\begin{equation}\label{eq:laplacian}
R_{\mathrm{lap}}[\phi\sharp\hat{\mu}^{\mathcal{M}_0}]=\sum_{i=1}^{|\mathcal{V}_0|}\frac{1}{|\mathcal{A}(i)|}\sum_{j\in\mathcal{A}(i)}\frac{1}{2}\|q^i_0-q^j_0+\Delta q^i-\Delta q^j\|^2,
\end{equation}
where $\mathcal{A}(i)$ denotes the adjacency list of the $i$-th vertex.

A commonly used discrepancy metric for non-rigid surface registration is the Chamfer distance \citep{fan2017chamfer}, which computes the bidirectional nearest distances with local point matching. The Chamfer distance between two probability measures is defined as:
\begin{equation}\label{eq:chamfer}
\begin{split}
D_{\mathrm{cham}}(\mu,\mu_*)
&=\frac{1}{2}\int_{\mathbb{R}^3}\inf_{x_*\in\mathrm{supp}(\mu_*)}\|x-x_*\|^2\mathrm{d}\mu(x)\\
&+\frac{1}{2}\int_{\mathbb{R}^3}\inf_{x\in\mathrm{supp}(\mu)}\|x_*-x\|^2\mathrm{d}\mu_*(x_*).
\end{split}
\end{equation}
Compared to the Chamfer distance, the SWD measures the global discrepancy between two probability distributions. Although the effectiveness of the SWD has been validated in non-rigid point cloud registration tasks \citep{lai2017swd,nguyen2021swd,nguyen2023swd,nguyen2024swd}, its performance for non-rigid surface registration problems, which require the preservation of mesh connectivity, remains to be further examined.

In this work, we propose a hybrid distributional optimisation objective, incorporating both the Chamfer and sliced Wasserstein discrepancies for non-rigid surface registration, formulated as:
\begin{equation}\label{eq:nonrigid_registration}
\begin{split}
\min_{\mu=\phi\sharp\mu^{\mathcal{M}_0}}F[\mu]
&=\lambda_{\mathrm{cham}}D_{\mathrm{cham}}(\mu,\mu^{\mathcal{M}_*})\\
&+\frac{\lambda_{\mathrm{sw}}}{2}SW_2^2(\mu,\mu^{\mathcal{M}_*})+\lambda_{\mathrm{lap}}R_{\mathrm{lap}}[\mu],
\end{split}
\end{equation}
where $\lambda_{\mathrm{cham}}, \lambda_{\mathrm{sw}}, \lambda_{\mathrm{lap}}$ are the weights for different terms. In particular, we let $\lambda_{\mathrm{cham}}=0$ and $\lambda_{\mathrm{sw}}=1$ at the beginning of the optimisation procedure to encourage non-local surface matching. This provides an ideal initial surface for the following local refinement. Subsequently, we let $\lambda_{\mathrm{cham}}=1$ and $\lambda_{\mathrm{sw}}=0$ to capture local surface correspondence. In our experiments, we empirically demonstrate that such a hybrid coarse-to-fine optimisation strategy yields significant improvement compared to using either the Chamfer distance or the SWD alone as the objective.

To solve the non-rigid surface registration problem (\ref{eq:nonrigid_registration}) using WGF, the Wasserstein gradient of the Chamfer distance $D_{\mathrm{cham}}(\mu,\mu_*)$ with respect to $\mu$ is derived by
\begin{equation}\label{eq:chamfer_grad}
\begin{split}
\nabla_W D_{\mathrm{cham}}&(\mu,\mu_*)(x)\\
&=x-\mathcal{T}_{\mu_*}(x)
+\int_{x=\mathcal{T}_{\mu}(x_*)}(x-x_*)~\mathrm{d}\mu_*(x_*),
\end{split}
\end{equation}
where
\begin{equation}\label{eq:chamfer_match}
\begin{split}
\mathcal{T}_{\mu_*}(x)&\in\arg\min_{x_*\in\mathrm{supp}(\mu_*)}\|x-x_*\|^2,\\
\mathcal{T}_{\mu}(x_*)&\in\arg\min_{x\in\mathrm{supp}(\mu)}\|x_*-x\|^2.
\end{split}
\end{equation}
For the mesh Laplacian regularisation $R_{\mathrm{lap}}$, the Wasserstein gradient is calculated as
\begin{equation}\label{eq:laplacian_grad}
\begin{split}
\nabla_W R_{\mathrm{lap}}[\phi\sharp\hat{\mu}^{\mathcal{M}_0}]&(q^i_0+\Delta q^i)\\
=&\frac{1}{|\mathcal{A}(i)|}\sum_{j\in\mathcal{A}(i)}(q^i_0-q^j_0+\Delta q^i-\Delta q^j).
\end{split}
\end{equation}

\section{AdamFlow for distributional optimisation}
Analogous to gradient-based optimisation in the Euclidean space, the distributional optimisation problem for surface registration defined in Eq.~(\ref{eq:registration}) can be solved by the WGF (\ref{eq:wgf}).
To further accelerate the optimisation procedure, we introduce \textit{AdamFlow}, a variant of the WGF with adaptive moment estimation for non-convex distributional optimisation. Firstly, based on the continuous form of the Adam optimisation method \citep{kingma2015adam, barakat2021convergence}, we formulate AdamFlow as a flow PDE that characterises the evolution of probability measures. Then, we prove that the solution of AdamFlow will converge asymptotically towards the set of critical points of the objective functional without convexity assumption. Finally, we present a particle-based approach for the numerical implementation of the proposed AdamFlow method.

\subsection{Formulation of AdamFlow}
\paragraph{Discrete Adam.}
We first introduce the classic Adam optimisation algorithm in the discrete setting \citep{kingma2015adam}. We consider an unconstrained optimisation problem $\min_{x\in\mathbb{R}^d} f(x)$ in the Euclidean space with objective function $f:\mathbb{R}^d\rightarrow\mathbb{R}$. For any vector $x\in\mathbb{R}^d$ and power $p\in\mathbb{R}$, we define $x^p:=[x_1^p,...,x_d^p]^{\top}$, where $x_i$ is the $i$-th entry of the vector $x$. For $k=0,...,K-1$, the classic Adam algorithm is expressed as the following iteration:
\begin{equation}\label{eq:adam_discrete}
\left\{\begin{aligned}
m_{k+1} & = \alpha m_{k} + (1-\alpha)\nabla f(x_{k}), \\
v_{k+1} & = \beta v_{k} + (1-\beta)(\nabla f(x_{k}))^2, \\
x_{k+1} & = x_{k}-\eta\frac{m_{k+1}/(1-\alpha^{k+1})}{\sqrt{v_{k+1}/(1-\beta^{k+1})}+\epsilon}, 
\end{aligned}\right.
\end{equation}
where $x_k$ denotes the variable to be optimised, $m_k$ and $v_k$ denote the biased estimation of the first and second moments of the gradient $\nabla f$, $\alpha,\beta\in[0,1)$ denote the decay rates of the exponential moving average respectively, $1-\alpha^{k+1}$ and $1-\beta^{k+1}$ denote the coefficients for bias correction, $\epsilon>0$ denotes a small number, and $\eta>0$ denotes the learning rate. For $k=0$, $x_0$ is an arbitrary initial value, and the estimated moments $m_0$ and $v_0$ are initialised to $0$.

\paragraph{Continuous Adam.}
Derived from the discrete Adam algorithm (\ref{eq:adam_discrete}), the continuous form of the Adam optimiser is formulated as the following dynamical system, known as Adam ODE \citep{barakat2021convergence}:
\begin{equation}\label{eq:adam_ode}
\left\{\begin{aligned}
\dot{m} & = (1-\alpha)(\nabla f(x)-m),  & m(0)&=0,\\
\dot{v} & = (1-\beta)((\nabla f(x))^2-v),  & v(0)&=0,\\
\dot{x} & = - \eta g_t(m,v),  & x(0)&=x_0,
\end{aligned} \right.
\end{equation}
where $t\in\mathbb{R}_+$, $x(t)\in\mathbb{R}^d$ describes the evolution of the optimisation variable, $m(t),v(t)\in\mathbb{R}^d$ are the estimation of the first and second moments of the gradient, and $g_t(m,v)\in\mathbb{R}^d$ is defined as
\begin{equation}\label{eq:gt}
g_t(m,v):=\frac{m/(1-e^{-(1-\alpha)t})}{\sqrt{v/(1-e^{-(1-\beta)t})}+\epsilon}, 
\end{equation}
where $1-e^{-(1-\alpha)t}$ and $1-e^{-(1-\beta)t}$ are the bias correction coefficients for the continuous exponential moving average, and $\eta>0$ is the learning rate. The initial value of Adam ODE (\ref{eq:adam_ode}) is set to $(x(0),m(0),v(0))=(x_0,0,0)$ to guarantee the existence of solutions \citep{barakat2021convergence}. Adam ODE defines a gradient flow with adaptive moment estimation for optimisation problems in Euclidean space.

\paragraph{AdamFlow.}
Inspired by the Hamiltonian flows \citep{chen2025accelerating} that introduce momentum terms to accelerate convex distributional optimisation, here we propose \textit{AdamFlow} with adaptive moment estimation of the Wasserstein gradient to perform distributional optimisation (\ref{eq:d_opt}) with a non-convex objective functional $F[\mu]$.

We extend the Adam method from the Euclidean space to the space of probability measures. Let the probability measure
\begin{equation*}
\begin{split}
\rho\in&\mathcal{P}_2(\mathbb{R}^d\times\mathbb{R}^d\times\mathbb{R}^d)\\
&:=\left\{\rho~\Big|\int_{\mathbb{R}^{3d}} \|x\|^2+\|m\|^2+\|v\|^2\mathrm{d}\rho(x,m,v)<\infty\right\}
\end{split}
\end{equation*}
describe the probability density of particles $(x,m,v)\in\mathbb{R}^{3d}$. We formulate AdamFlow as a flow PDE over the probability measure such that the evolution of the particles $(x,m,v)$ follows the same dynamics as the Adam ODE system (\ref{eq:adam_ode}). More precisely, the AdamFlow PDE is defined as:
\begin{equation}\label{eq:adamflow}
\begin{split}
\partial_t\rho_t=&-\nabla_m\cdot\left(\rho_t(1-\alpha)\left(\nabla_W F[\mu_t]-m\right)\right)\\
&-\nabla_v\cdot\left(\rho_t(1-\beta)\left((\nabla_W F[\mu_t])^2-v\right)\right)\\
&+\nabla_x\cdot\left(\rho_t\eta g_t(m,v)\right),\\
\end{split}
\end{equation}
where $\rho_t:\mathbb{R}_+\rightarrow\mathcal{P}_2(\mathbb{R}^d\times\mathbb{R}^d\times\mathbb{R}^d)$ is a curve that describes the evolution of the probability measure, $\mu_t=\rho_t^X:=\rho_t(\cdot,\mathbb{R}^d,\mathbb{R}^d)$ is the $x$-marginal of $\rho_t$, $g_t(m,v)$ is defined in Eq.~(\ref{eq:gt}), $\nabla_W F:\mathbb{R}^d\rightarrow\mathbb{R}^d$ is the Wasserstein gradient of the objective functional $F$, and $(\nabla_W F[\mu])^2(x):=(\nabla_W F[\mu](x))^2$. The parameters $\alpha,\beta\in[0,1)$ are the decay rates of the exponential moving average, and $\eta>0$ is the learning rate. We assume the initial condition $\rho_0$ of the AdamFlow (\ref{eq:adamflow}) satisfies $\rho_0(x,m,v)=\mu_0(x)\delta_{(0,0)}(m,v)$, where $\mu_0$ is the initial $x$-marginal and $\delta_{(0,0)}$ is a Dirac measure at $(0,0)\in\mathbb{R}^d\times\mathbb{R}^d$. In other words, the initial estimations of the first and second moments of the Wasserstein gradient $\nabla_W F[\mu]$ satisfy $m=v=0$ a.e. on $\mathbb{R}^d$.

AdamFlow is constructed such that $x$-marginal $\mu_t$ of the solution $\rho_t$ of the flow PDE (\ref{eq:adamflow}) converges to the set of critical points of the objective functional $F$. The curve $\rho_t$ evolves particles $(x,m,v)$ following the same dynamics as the Adam method, such that $m$ and $v$ adaptively estimate the first and second moments of the Wasserstein gradient $\nabla_W F[\mu]$. 
In fact, if we define an object functional $F_f[\mu]:=\int_{\mathbb{R}^d} f(x)\mathrm{d}\mu(x)$, then the AdamFlow PDE (\ref{eq:adamflow}) becomes the corresponding continuity equation derived from the Adam ODE system (\ref{eq:adam_ode}), which implies that the solution $\rho_t$ of AdamFlow indeed evolves the particles $(x,m,v)$ along the trajectories defined by the Adam ODE (\ref{eq:adam_ode}). More specifically, analogous to the Hamiltonian flow method \citep{chen2025accelerating}, we have the following propositions:

\begin{proposition}\label{prop1}
If $(x(t),m(t),v(t))$ is a solution of Adam ODE (\ref{eq:adam_ode}) with an objective function $f$, then the trajectory of the Dirac measure $\delta_{(x(t),m(t),v(t))}$ is a solution of AdamFlow (\ref{eq:adamflow}) with the objective functional $F_f[\mu]=\int_{\mathbb{R}^d}f(x)\mathrm{d}\mu(x)$ in the distributional sense.
\end{proposition}

\begin{proposition}\label{prop2}
If $x(t)$ converges to a critical point of an objective function $f$, then the trajectory of the Dirac measure $\delta_{x(t)}$ converges to a critical point of the objective functional $F_f[\mu]=\int_{\mathbb{R}^d} f(x)\mathrm{d}\mu(x)$. 
\end{proposition}

Proposition \ref{prop1} and \ref{prop2} indicate the equivalence between the Adam ODE (\ref{eq:adam_ode}) and AdamFlow (\ref{eq:adamflow}), when the objective functional is defined as a potential energy $F_f[\mu]$ with potential function $f$. Therefore, AdamFlow can be viewed as a generalisation of the Adam ODE system (\ref{eq:adam_ode}) from the Euclidean space to the space of probability measures, while the objective is extended from the potential energy $F_f$ to an arbitrary functional $F$.

\subsection{Convergence analysis}
In this section, we derive the asymptotic convergence of AdamFlow (\ref{eq:adamflow}). Instead of finding a global minimum for a convex distributional optimisation problem \citep{chen2025accelerating}, we show that the solutions of AdamFlow will converge to the set of critical points of the objective functional without the requirement of convexity.

\begin{theorem}\label{thm1}
Assume the objective functional is lower bounded by $F[\mu]\geq\underline{F}$. Let $\rho_t\in\mathcal{P}_2(\mathbb{R}^d\times\mathbb{R}^d\times\mathbb{R}^d)$ be a solution of AdamFlow (\ref{eq:adamflow}) with $x$-marginal $\mu_t$ and initial condition $\rho_0(x,m,v)=\mu_0(x)\delta_{(0,0)}(m,v)$. If $\rho_t$ is absolute continuous and the parameters satisfy $4\alpha-\beta<3$, then every limit point of the curve $\mu_t$ is a critical point of the objective functional $F[\mu]$. More specifically, let $\omega(\mu_0)$ denote an $\omega$-limit set that contains all limit points of $\mu_t$:
\begin{equation}\label{eq:omega_set}
\omega(\mu_0):=\left\{\mu_{\infty}~|~\mu_{\infty}=\rho_{\infty}^X,~\exists t_n\rightarrow\infty\mbox{ such that }\rho_{t_n}\rightarrow\rho_{\infty}\right\}.
\end{equation}
Then $\omega(\mu_0)$ is a subset of the set of critical points of $F[\mu]$, i.e., 
$\omega(\mu_0)\subseteq\left\{\mu_*|\nabla_W F[\mu_*](x)=0,~\mbox{a.e.}~ x\in\mathrm{supp}(\mu_*)\right\}$.
\end{theorem}

The detailed proof of Theorem \ref{thm1} is provided in Appendix \ref{appendix_A3} using Lyapunov approach \citep{barakat2021convergence,chen2025accelerating}. Theorem \ref{thm1} demonstrates the asymptotic convergence of the solution of AdamFlow (\ref{eq:adamflow}), while the convergence rate of AdamFlow could be further quantified by introducing additional assumptions such as convexity or Łojasiewicz property of the objective functional.

\subsection{Numerical implementation}
AdamFlow defines a flow PDE (\ref{eq:adamflow}) which can be solved numerically by the JKO scheme \citep{jordan1998jko}, numerical PDE solvers \citep{ames2014numerical}, or particle-based approximation \citep{liu2016stein,liu2019particle,chewi2024sot,chen2025accelerating}. However, both JKO scheme and numerical PDE solvers are computationally expensive. The JKO scheme needs to solve a sequence of distributional optimisation problems to compute the entire trajectory of $\rho_t$, and numerical PDE solvers require discretisation over the spatio-temporal domain $\mathbb{R}_+\times\mathbb{R}^d\times\mathbb{R}^d\times\mathbb{R}^d$.

Therefore, in line with the Hamiltonian flow \citep{chen2025accelerating}, we adopt particle-based method to solve AdamFlow (\ref{eq:adamflow}) and approximate the solution $\rho_t$ by an empirical measure $\hat{\rho}_t$ with $N$ particles $(X^i,M^i,V^i)$:
\begin{equation}\label{eq:empirical}
\rho_t\approx\hat{\rho}_t:=\frac{1}{N}\sum_{i=1}^N\delta_{(X^i(t),M^i(t),V^i(t))}.
\end{equation}
Its $x$-marginal is defined by $\hat{\mu}_t=\hat{\rho}_t^X:=\frac{1}{N}\sum_{i=1}^N\delta_{X^i(t)}$. Such a particle-based approximation of AdamFlow leads to the following result:
\begin{proposition}\label{prop3}
Suppose the solution $\rho_t$ of AdamFlow (\ref{eq:adamflow}) is approximated by the empirical measure $\hat{\rho}_t$ defined in (\ref{eq:empirical}). Then, the particle $(X^i,M^i,V^i)$ evolves according to the following ODE system:
\begin{equation}\label{eq:adamflow_particle}
\left\{\begin{aligned}
\dot{M}^i & = (1-\alpha)\left(\nabla_W F[\hat{\mu}_t](X^i)-M^i\right), \\
\dot{V}^i & = (1-\beta)\left((\nabla_W F[\hat{\mu}_t](X^i))^2-V^i\right), \\
\dot{X}^i & = -\eta g_t(M^i,V^i),
\end{aligned} \right.
\end{equation}
for $i=1,...,N$, where $(X^i(0),M^i(0),V^i(0))=(X_0^i,0,0)$ is the initial value of the particle.
\end{proposition}

The particle dynamical system (\ref{eq:adamflow_particle}) can be integrated by any ODE solvers. For simplicity, we employ the Euler method with step size $h$. The final numerical algorithm of AdamFlow is expressed as:
\begin{equation}\label{eq:adamflow_euler}
\left\{
\begin{aligned}
M^i_{k+1} & = M^i_{k}+h(1-\alpha)\left(\nabla_W F[\hat{\mu}_{k}](X^i_{k})-M^i_{k}\right), \\
V^i_{k+1} & = V^i_{k}+h(1-\beta)\left((\nabla_W F[\hat{\mu}_{k}](X^i_{k}))^2-V^i_{k}\right),\\
X^i_{k+1} & = X^i_{k}-h\eta g_{h(k+1)}(M^i_{k+1},V^i_{k+1}),
\end{aligned} \right.
\end{equation}
for $i=1,...,N$ and $k=0,...,K-1$, where the empirical measure $\hat{\mu}_k:=\frac{1}{N}\sum_{i=1}^N\delta_{X^i_k}$. The particle-based AdamFlow algorithm (\ref{eq:adamflow_euler}) is computationally efficient since all particles can be updated in parallel.

\begin{algorithm}[h]
\caption{AdamFlow for affine surface registration}\label{alg:affine}
\begin{algorithmic}[1] 
\State \textbf{Input:} 
source mesh $\mu_0=\mu^{\mathcal{M}_0}$ with vertices $q^i_0\in\mathcal{V}_0$, target mesh $\mu_*=\mu^{\mathcal{M}_*}$, number of vertices $N=|\mathcal{V}_0|$, integration steps $K$, step size $h$, number of projects $L$, decay rates $\alpha,\beta\in[0,1)$, small number $\epsilon>0$, learning rate $\eta$, initial affine matrix $A_0=I_{3\times 3}$, initial translation vector $b_0=\mathbf{0}_{3\times 1}$, initial moments $m_0^{A}=v_0^{A}=\mathbf{0}_{3\times 3}$ and $m_0^{b}=v_0^{b}=\mathbf{0}_{3\times 1}$.
\For{$k=0,...,K-1$} 
\State Sample points on the fixed mesh $x_*^1,...,x_*^{N}\sim\mu_*$
\State $x^i_k=A_kq^i_0+b_k$ for $i=1,...,N$
\State $\hat{\mu}_k=\frac{1}{N}\sum_{i=1}^{N}\delta_{x^i_k}$ and $\hat{\mu}_*=\frac{1}{N}\sum_{i=1}^{N}\delta_{x^i_*}$

\State Sample projection directions $\theta_1,...,\theta_L\sim\mathbb{S}^2$
\State Compute $\nabla_W F[\hat{\mu}_k]=\frac{1}{2}\nabla_W SW_2^2(\hat{\mu}_k, \hat{\mu}_*)$

\State $\mathrm{grad}~A_k=\frac{1}{N}\sum_{i=1}^N \nabla_W F[\hat{\mu}_k](x^i_k)(q^i_0)^{\top} $
\State $\mathrm{grad}~b_k=\frac{1}{N}\sum_{i=1}^N \nabla_W F[\hat{\mu}_k](x^i_k)$

\State $m^A_{k+1} = m^A_{k}+h(1-\alpha)\left(\mathrm{grad}~A_k-m^A_{k}\right)$
\State $m^b_{k+1} = m^b_{k}+h(1-\alpha)\left(\mathrm{grad}~b_k-m^b_{k}\right)$
\State $v^A_{k+1} = v^A_{k}+h(1-\beta)\left((\mathrm{grad}~A_k)^2-v^A_{k}\right)$
\State $v^b_{k+1} = v^b_{k}+h(1-\beta)\left((\mathrm{grad}~b_k)^2-v^b_{k}\right)$
\State $A_{k+1} = A_{k}-h\eta g_{h(k+1)}(m^A_{k+1},v^A_{k+1})$
\State $b_{k+1} = b_{k}-h\eta g_{h(k+1)}(m^b_{k+1},v^b_{k+1})$
\EndFor

\State \textbf{Return:} $A_K$ and $b_K$.
\end{algorithmic}
\end{algorithm}

\begin{algorithm}[h]
\caption{AdamFlow for non-rigid surface registration}\label{alg:nonrigid}
\begin{algorithmic}[1] 
\State \textbf{Input:} 
source mesh $\mu_0=\mu^{\mathcal{M}_0}$ with vertices $q^i_0\in\mathcal{V}_0$, target mesh $\mu_*=\mu^{\mathcal{M}_*}$, number of vertices $N=|\mathcal{V}_0|$, integration steps $K_{\mathrm{sw}}$ and $K_{\mathrm{cham}}$, step size $h$, number of projects $L$, regularisation weight $\lambda_{\mathrm{lap}}$, decay rates $\alpha,\beta\in[0,1)$, small number $\epsilon>0$, learning rate $\eta_{\mathrm{sw}}$ and $\eta_{\mathrm{cham}}$, initial displacement $\Delta q^i_0=0$, initial moments $m_0^i=v_0^i=0$.

\State $x^i_0 = q^i_0+\Delta q^i_0$ for $i=1,...,N$
\State $\lambda_{\mathrm{sw}}=1$, $\lambda_{\mathrm{cham}}=0$, $
\eta=\eta_{\mathrm{sw}}$, $K=K_{\mathrm{sw}}+K_{\mathrm{cham}}$
\For{$k=0,...,K-1$} 
\If{$k=K_{\mathrm{sw}}$}
\State $\lambda_{\mathrm{sw}}=0$, $\lambda_{\mathrm{cham}}=1$, $\eta=\eta_{\mathrm{cham}}$, $m_k^i=v_k^i=0$
\EndIf
\State Sample points on the fixed mesh $x_*^1,...,x_*^{N}\sim\mu_*$
\State $\hat{\mu}_k=\frac{1}{N}\sum_{i=1}^{N}\delta_{x^i_k}$ and $\hat{\mu}_*=\frac{1}{N}\sum_{i=1}^{N}\delta_{x^i_*}$

\If{$k<K_{\mathrm{sw}}$}
\State Sample projection directions $\theta_1,...,\theta_L\sim\mathbb{S}^2$
\EndIf 
\State Compute $\nabla_W F[\hat{\mu}_k]$ defined in Eq.~(\ref{eq:nonrigid_registration})

\State $m^i_{k+1} = m^i_{k}+h(1-\alpha)\left(\nabla_W F[\hat{\mu}_{k}](x^i_{k})-m^i_{k}\right)$
\State $v^i_{k+1} = v^i_{k}+h(1-\beta)\left((\nabla_W F[\hat{\mu}_{k}](x^i_{k}))^2-v^i_{k}\right)$
\State $x^i_{k+1} = x^i_{k}-h\eta g_{h(k+1)}(m^i_{k+1},v^i_{k+1})$
\EndFor
\State \textbf{Return:} $\Delta q^i_K=x^i_K-q^i_0$ for $i=1,...,N$.
\end{algorithmic}
\end{algorithm}

\subsection{AdamFlow for surface registration}
The particle-based implementation of the AdamFlow enables us to effectively solve both affine and non-rigid surface registration problems, defined in Eq.~(\ref{eq:affine_registration}) and (\ref{eq:nonrigid_registration}) respectively. We provide particle-based approximation for the source mesh $\mu^{\mathcal{M}_0}$ and the target mesh $\mu^{\mathcal{M}_*}$. The source mesh $\mu^{\mathcal{M}_0}$ is approximated by an empirical measure of its vertices as described in Eq.~(\ref{eq:mesh_empirical}). For the target mesh, we sample $N=|\mathcal{V}_0|$ points on the mesh for each time step, \emph{i.e.}, $x_*^1,...,x_*^{N}\sim\mu^{\mathcal{M}_*}$, and approximate $\mu^{\mathcal{M}_*}$ by an empirical measure $\hat{\mu}^{\mathcal{M}_*}=\frac{1}{N}\sum_{i=1}^{N}\delta_{x^i_*}$.

For affine surface registration, we initialise an identity affine matrix $A_0=I_{3\times 3}$ as a zero translation vector $b_0=\mathbf{0}_{3\times 1}$, and update the affine matrix and translation vector simultaneously based on the chain rule using AdamFlow. The detailed procedure is provided in Algorithm~\ref{alg:affine}. For non-rigid surface registration, we initialise the displacement of each vertex to $\Delta q^i_0=0$ and define $x^i_{k}=q^i_0+\Delta q^i_k$ for each time step $k$ and vertex $i=1,...,N$. All vertices are optimised in parallel. As discussed in Section \ref{sec:nonrigid}, we perform non-rigid surface registration in a coarse-to-fine manner. Specifically, we first optimise the SWD for $K_{\mathrm{sw}}$ steps for global surface alignment, and then we minimise the Chamfer distance for $K_{\mathrm{cham}}$ steps for local surface refinement. The optimisation scheme is summarised in Algorithm~\ref{alg:nonrigid}.

\section{Experiments}
In this section, we empirically verify the effectiveness of the sliced Wasserstein objective with AdamFlow optimisation for both affine and non-rigid surface registration tasks. Particularly, we demonstrate the superiority of the proposed objective functional and optimisation method:
\begin{itemize}[leftmargin=*]
\item \textit{Objective Functional.} The SWD significantly improves surface registration accuracy and computational efficiency compared to registration via local correspondence only.
\item \textit{Optimisation Method.} AdamFlow exhibits superior convergence property compared to accelerated WGF \citep{chen2025accelerating} for non-convex surface registration problems.
\end{itemize}

\subsection{Experiment settings}
\paragraph{Dataset.} To evaluate the performance of AdamFlow on surface registration problems, we consider three anatomical structures: the liver, the pancreas, and the left ventricle of the heart. The surface meshes of three types of organs are extracted from two datasets, AbdomenCT-1K \citep{ma2021abdomenct} and ImageCAS \citep{zeng2023imagecas}. For the AbdomenCT-1K dataset, we collect 360 ground-truth segmentation maps of abdominal CT scans that include the liver and pancreas. For the ImageCAS dataset, we employ TotalSegmentator v2.0 \citep{wasserthal2023totalsegmentator} to create cardiac four-chamber segmentations, including the left ventricle, for 500 coronary computed tomography angiography (CCTA) images. The surface meshes are extracted by the marching cubes algorithm \citep{lorensen1998marching} following Laplacian mesh smoothing. The dataset information and the resolution of extracted meshes are summarised in Table~\ref{tab:dataset}. For each organ, we randomly select 300 pairs of source and target meshes to evaluate the registration methods. A small set of 20 random pairs of meshes are used for hyperparameter selection.

\begin{table}
\setlength{\tabcolsep}{4pt}
\centering
\caption{Dataset information for surface registration. The number of meshes and the average numbers of vertices for the meshes of each anatomical structure are reported.}
\begin{tabular}{llcc}
\toprule
Dataset & Organ & \#Meshes & \#Vertices\\
\midrule
\multirow{2}{*}{AbdomenCT-1K} & Liver & 360 & $34,553\pm6,229$ \\
& Pancreas & 360 & $6,194\pm1,889$ \\
\midrule
ImageCAS & Left ventricle & 500 & $10,420\pm1,551$ \\
\bottomrule
\end{tabular}
\label{tab:dataset}
\end{table}

\begin{figure*}[h]
\centering
\begin{subfigure}{.49\textwidth}
  \centering
  \includegraphics[width=1.0\linewidth]{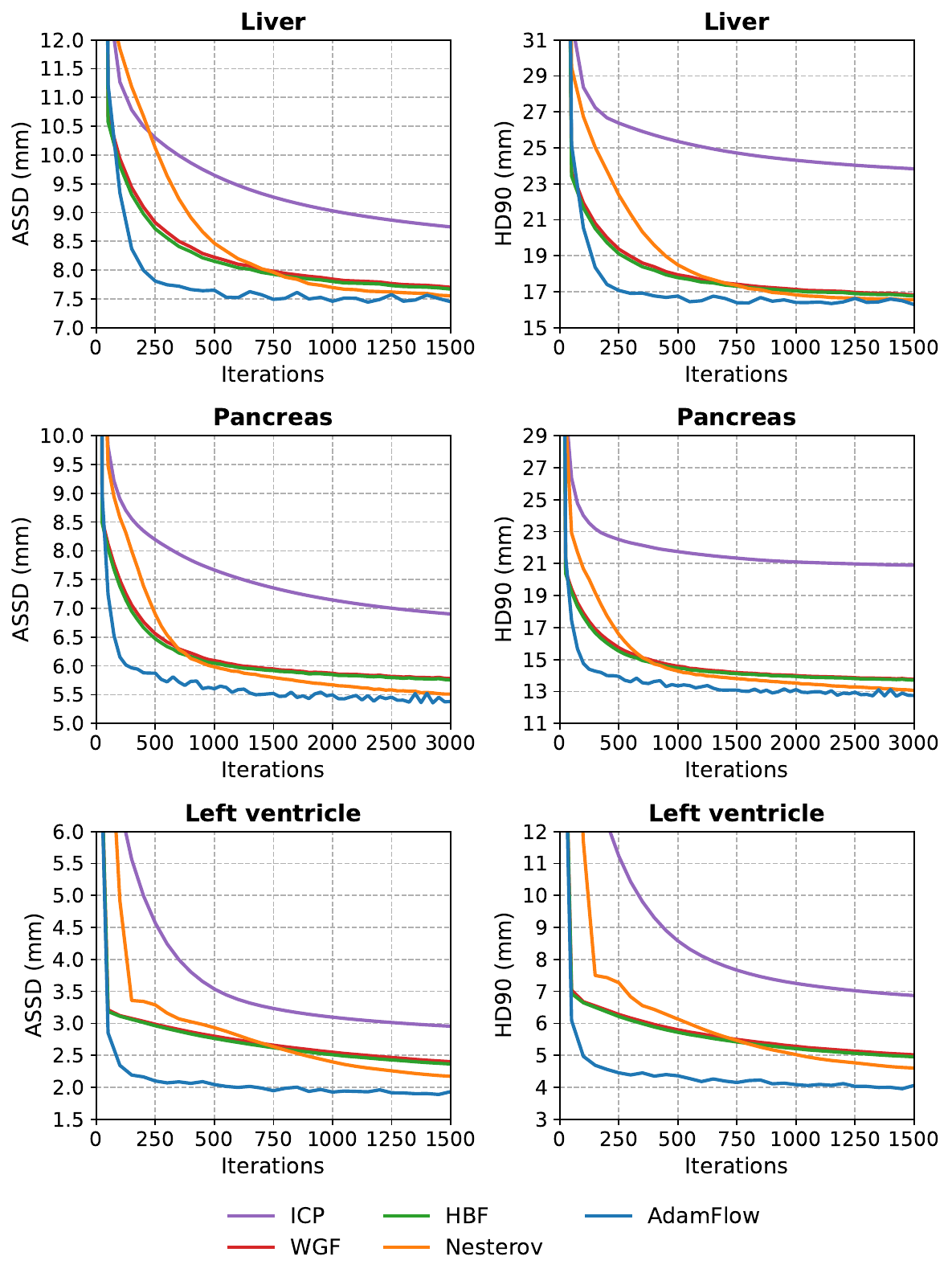}
  \caption{Comparative results for affine surface registration.}
  \label{fig:affine}
\end{subfigure}%
\hfill
\begin{subfigure}{.49\textwidth}
  \centering
  \includegraphics[width=1.0\linewidth]{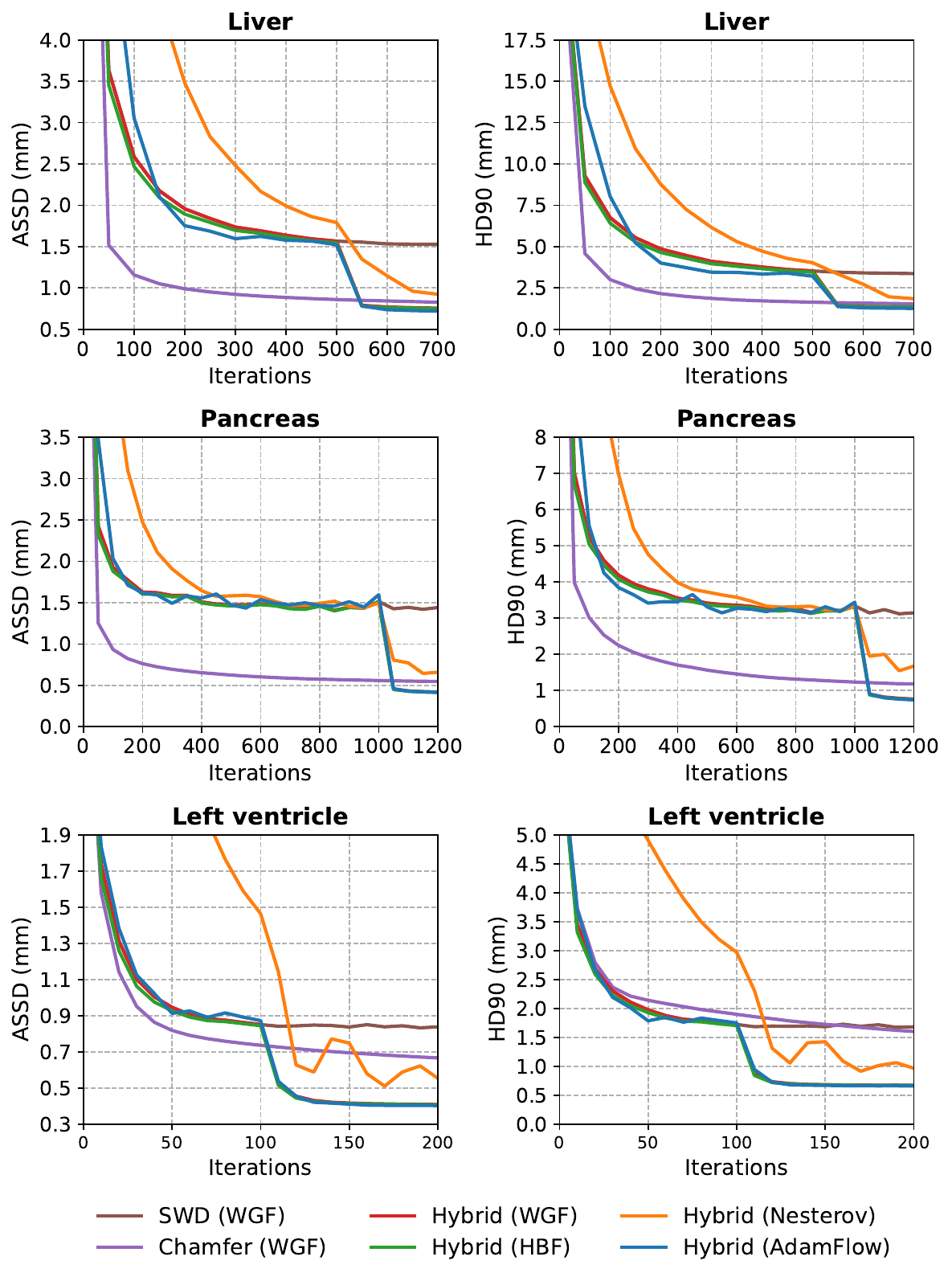}
  \caption{Comparative results for non-rigid surface registration.}
  \label{fig:nonrigid}
\end{subfigure}
\caption{Comparative results and convergence curves for affine and non-rigid surface registration. The registration errors are measured by ASSD (mm) and HD90 (mm). Registration errors decrease with increasing numbers of optimisation iterations.}
\end{figure*}

\paragraph{Comparison of objective functionals.} We compare across different objective functionals for surface registration problems, in terms of registration error and computational efficiency. For affine surface registration, we compare the SWD objective (\ref{eq:affine_registration}) with the ICP algorithm \citep{besl1992icp} that optimises local point correspondence (\ref{eq:icp}). For non-rigid surface registration, we compare the proposed hybrid objective functional (\ref{eq:nonrigid_registration}) against objectives based solely on the SWD (\ref{eq:sliced_wasserstein_mc}) or the Chamfer distance (\ref{eq:chamfer}). Note that registration methods such as the thin plate spline \citep{bookstein2002tps,chui2003tps} and the coherent point drift \citep{myronenko2006cpd,myronenko2010cpd} are not compared, due to their high computational complexity, \emph{i.e.}, $\mathcal{O}(N^3)$ for $N$ vertices.

\paragraph{Comparison of optimisation methods.}
For surface registration problems, we compare the AdamFlow optimisation method with other variants of WGF \citep{chen2025accelerating}, including the original WGF (\ref{eq:wgf}), the heavy-ball flow (HBF), and the Nesterov flow. Let $\rho_t:\mathbb{R}_+\rightarrow\mathcal{P}_2(\mathbb{R}^3\times\mathbb{R}^3)$ be a curve of probability measure, the HBF is defined as
\begin{equation}\label{eq:hbf}
\partial_t\rho_t=-\nabla_x\cdot(\rho_t m)+\nabla_m\cdot\left(\mu_t\left(am+\nabla_W F[\mu_t]\right)\right),
\end{equation}
where $m\in\mathbb{R}^3$ is the momentum term and the hyperparameter $a>0$ is set to 0.9 in our experiments. Similarly, the Nesterov flow is defined as
\begin{equation}\label{eq:nesterov}
\partial_t\rho_t=-\nabla_x\cdot(\rho_t m)+\nabla_m\cdot\left(\mu_t\left(3m/t+\nabla_W F[\mu_t]\right)\right).
\end{equation}
The exponentially convergent variational
acceleration flow proposed by \cite{chen2025accelerating} is not compared since it is numerically unstable.

\begin{table}
\setlength{\tabcolsep}{12pt}
\centering
\caption{The integration steps used for affine and non-rigid registration of different anatomical surfaces.}
\begin{tabular}{l|c|cc}
\toprule
& Affine & \multicolumn{2}{c}{Non-rigid} \\
Organ & $K_{\mathrm{affine}}$ & $K_{\mathrm{sw}}$ & $K_{\mathrm{cham}}$\\
\midrule
Liver & $1,500$ & $500$ & $200$ \\
Pancreas & $3,000$ & $1,200$ & $200$ \\
Left ventricle & $1,500$ & $100$ & $100$ \\
\bottomrule
\end{tabular}
\label{tab:integration_step}
\end{table}

\begin{table}
\setlength{\tabcolsep}{12pt}
\centering
\caption{The learning rates used in different optimisation methods for affine and non-rigid surface registration.}
\begin{tabular}{l|c|cc}
\toprule
& Affine & \multicolumn{2}{c}{Non-rigid} \\
Method & $\eta_{\mathrm{affine}}$ & $\eta_{\mathrm{sw}}$ & $\eta_{\mathrm{cham}}$\\
\midrule
ICP & $10^{-6}$ & -- & -- \\
WGF & $10^{-5}$ & $0.5$ & $0.1$ \\
HBF & $10^{-5}$ & $0.5$ & $0.1$ \\
Nesterov & $10^{-7}$ & $0.005$ & $0.005$ \\
AdamFlow & $10^{-2}$ & $0.5$ & $0.1$ \\
\bottomrule
\end{tabular}
\label{tab:learning_rate}
\end{table}

\paragraph{Evaluation metrics.} 
To evaluate the performance of surface registration, we measure the absolute symmetric surface distance (ASSD) and 90th percentile of Hausdorff distance (HD90) \citep{cruz2021deepcsr} between the transformed and target meshes. We randomly sample 50,000 points on both meshes and compute their distances as registration errors. To measure the computational efficiency and runtime of the algorithms, all optimisation algorithms are implemented using PyTorch \citep{paszke2019pytorch} and PyTorch3D \citep{ravi2020pytorch3d} with GPU acceleration. All experiments are conducted on an 8-core Intel Core i7-11700K CPU and a NVIDIA GeForce RTX 3080 GPU with 12GB memory.

\paragraph{Hyperparameters.} We provide the hyperparameters of objective functionals and optimisation methods used in comparative experiments. For the SWD, the number of Monte Carlo projections is set to $L=4$ for efficiency. The mesh Laplacian regularisation weight in the non-rigid surface registration problem (\ref{eq:nonrigid_registration}) is set to $\lambda_{\mathrm{lap}}=2.0$. For AdamFlow (\ref{eq:adamflow}), we let $\epsilon=10^{-10}$ and decay rates $\alpha=0.9$, $\beta=0.95$. For AdamFlow and all baseline methods, we use the forward Euler method to integrate the particle-based ODEs with the same step size $h=1$. As reported in Table~\ref{tab:integration_step}, we adopt different integration steps $K$ for different anatomical structures due to geometric variations. To ensure generalisability and fair comparison, for each optimisation algorithm, we use the same learning rate for all three types of organs as listed in Table~\ref{tab:learning_rate}. The hyperparameter selection experiments and detailed discussions are provided in Section \ref{sec:hyperparameter}.

\begin{table*}
\caption{Comparison of AdamFlow to different optimisation methods for affine surface registration. The registration errors are measured by ASSD (mm) and HD90 (mm). The best results are bolded. The symbol * indicates that AdamFlow achieves significantly better performance than all baseline methods (paired $t$-test, $p<0.05$).}
\begin{tabular}{l|ccc|ccc}
\toprule
& \multicolumn{3}{c|}{ASSD (mm) $\downarrow$} & \multicolumn{3}{c}{HD90 (mm) $\downarrow$} \\
Method & Liver & Pancreas & Left ventricle & Liver & Pancreas & Left ventricle \\
\midrule
ICP &
$8.751\pm3.118$ & $6.902\pm3.453$ & $2.955\pm0.956$ & 
$23.83\pm13.48$ & $20.90\pm15.25$ & $6.870\pm3.382$ \\
WGF &
$7.701\pm1.953$ & $5.780\pm1.724$ & $2.401\pm0.650$ & 
$16.82\pm4.715$ & $13.77\pm4.951$ & $5.016\pm1.764$ \\
HBF &
$7.670\pm1.947$ & $5.757\pm1.715$ & $2.365\pm0.644$ & 
$16.76\pm4.710$ & $13.71\pm4.926$ & $4.953\pm1.758$ \\
Nesterov &
$7.555\pm1.926$ & $5.511\pm1.649$ & $2.172\pm0.631$ & 
$16.55\pm4.691$ & $13.07\pm4.662$ & $4.602\pm1.733$\\
AdamFlow &
$~\mathbf{7.452}\pm\mathbf{1.880}^*$ & $~\mathbf{5.381}\pm\mathbf{1.451}^*$ & $~\mathbf{1.932}\pm\mathbf{0.529}^*$ &  
$~\mathbf{16.28}\pm\mathbf{4.597}^*$ & $~\mathbf{12.75}\pm\mathbf{4.289}^*$ & $~\mathbf{4.059}\pm\mathbf{1.354}^*$ \\
\bottomrule
\end{tabular}
\label{tab:affine}
\end{table*}

\begin{table*}
\setlength{\tabcolsep}{4pt}
\caption{Comparison of different objective functionals and optimisation methods for non-rigid surface registration. The registration errors are measured by ASSD (mm) and HD90 (mm). The best results are bolded. The symbol * indicates that AdamFlow achieves significantly better performance than all baseline methods (paired $t$-test, $p<0.05$).}
\begin{tabular}{ll|ccc|ccc}
\toprule
& & \multicolumn{3}{c|}{ASSD (mm) $\downarrow$} & \multicolumn{3}{c}{HD90 (mm) $\downarrow$} \\
Objective & Method & Liver & Pancreas & Left ventricle & Liver & Pancreas & Left ventricle \\
\midrule
SWD & WGF &
$1.528\pm0.394$ & $1.441\pm0.487$ & $0.838\pm0.172$ &   $3.369\pm1.550$ & $3.148\pm1.427$ & $1.680\pm0.471$\\
Chamfer & WGF &
$0.827\pm0.264$ & $0.548\pm0.211$ & $0.666\pm0.324$ & 
$1.537\pm1.362$ & $1.184\pm0.959$ & $1.600\pm1.102$  \\
Hybrid & WGF &
$0.755\pm0.192$ & $0.419\pm0.144$ & $0.408\pm0.052$ & 
$1.346\pm1.101$ & $0.763\pm0.596$ & $0.670\pm0.157$ \\
Hybrid & HBF &
$0.749\pm0.189$ & $0.416\pm0.141$ & $0.407\pm0.051$ & 
$1.337\pm1.104$ & $0.747\pm0.558$ & $0.668\pm0.152$\\
Hybrid & Nesterov &
$0.923\pm0.243$ & $0.659\pm0.874$ & $0.553\pm0.091$ &
$1.848\pm1.151$ & $1.674\pm4.632$ & $0.965\pm0.280$ \\
Hybrid & AdamFlow &
$~\mathbf{0.720}\pm\mathbf{0.145}^*$ & $\mathbf{0.415}\pm\mathbf{0.144}$ & $~\mathbf{0.402}\pm\mathbf{0.049}^*$ & 
$~\mathbf{1.264}\pm\mathbf{0.833}^*$ & $\mathbf{0.742}\pm\mathbf{0.553}$ & $~\mathbf{0.658}\pm\mathbf{0.113}^*$ \\
\bottomrule
\end{tabular}
\label{tab:nonrigid}
\end{table*}

\subsection{Comparative results}
\paragraph{Affine surface registration.} For affine surface registration, we run AdamFlow following the procedure in Algorithm~\ref{alg:affine}. The comparative results are presented in Table~\ref{tab:affine}, and the optimisation curves are depicted in Figure~\ref{fig:affine}. Table~\ref{tab:affine} shows that for all three anatomical structures, our AdamFlow optimisation method achieves significantly better registration performance compared to all baseline methods, in terms of ASSD and HD90 errors. In addition, the SWD effectively captures global shape alignment, leading to substantial improvement compared to the ICP method that minimises local nearest neighbour distances. The qualitative comparisons between the ICP method and the AdamFlow method with the SWD objective are visualised in Figure~\ref{fig:mesh_affine}. Figure~\ref{fig:affine} demonstrates that AdamFlow exhibits a faster convergence rate and better generalisability than all baseline methods, as it can adaptively adjust the learning rate. AdamFlow starts to converge after only 30\% of the total integration steps for all three types of organ surfaces using the same learning rate.

\begin{figure}
\centering
\includegraphics[width=1.0\linewidth]{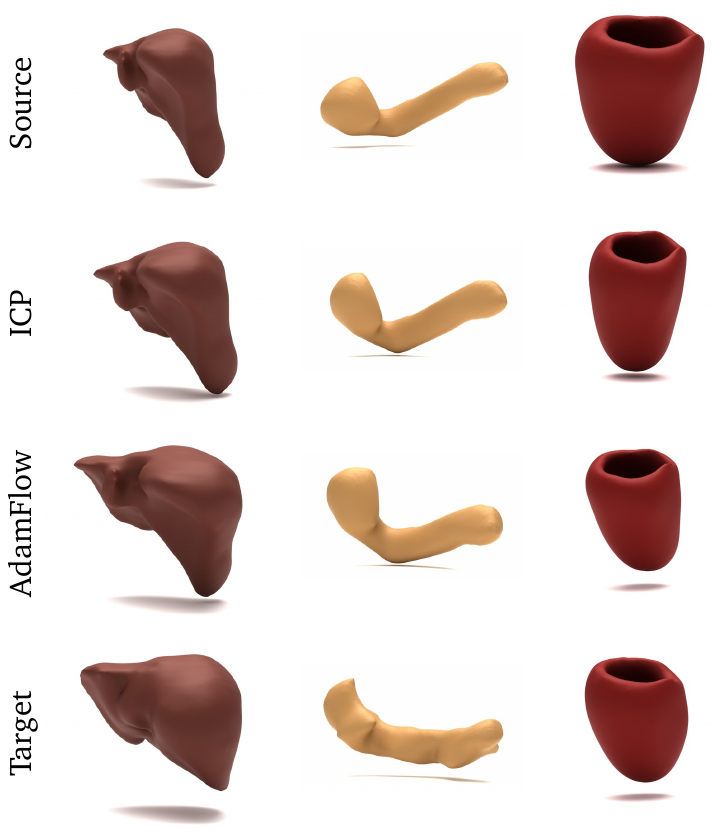}
\caption{Qualitative comparisons for affine surface registration of the liver, pancreas, and left ventricle from a source mesh to a target mesh. AdamFlow shows consistently better affine registration accuracy than the traditional ICP algorithm.}
\label{fig:mesh_affine}
\end{figure}

\paragraph{Non-rigid surface registration.} Usually, the non-rigid surface registration is performed after affine alignment. To reduce the computational cost, we only perform centre of mass alignment between the moving and target surface meshes before non-rigid surface registration. Similar to the procedure of Algorithm~\ref{alg:nonrigid}, all baseline approaches and AdamFlow optimise the hybrid objective functional (\ref{eq:nonrigid_registration}) in a coarse-to-fine manner, \emph{i.e.}, minimising the SWD for $K_{\mathrm{sw}}$ steps and then the Chamfer distance for $K_{\mathrm{cham}}$ steps. The comparative results are provided in Table~\ref{tab:nonrigid}, the error curves are plotted in Figure~\ref{fig:nonrigid}, and the illustration of coarse-to-fine registration is presented in Figure~\ref{fig:mesh_nonrigid}. 

Table~\ref{tab:nonrigid} demonstrates that AdamFlow achieves the best registration accuracy compared to all baseline approaches for all three types of organs, showing significant improvement in the surface registration of the liver and left ventricle. Table~\ref{tab:nonrigid} and Figure~\ref{fig:nonrigid} indicate that the SWD only captures global shape alignment without fine details. Using the Chamfer distance alone yields suboptimal results due to its sensitivity to initialisation. Our hybrid objective functional takes advantage of both SWD and Chamfer distance. The SWD with global distribution matching provides an ideal initialisation for the subsequent local refinement using the Chamfer distance, resulting in superior performance in non-rigid surface registration tasks.

\begin{figure}
\centering
\includegraphics[width=1.0\linewidth]{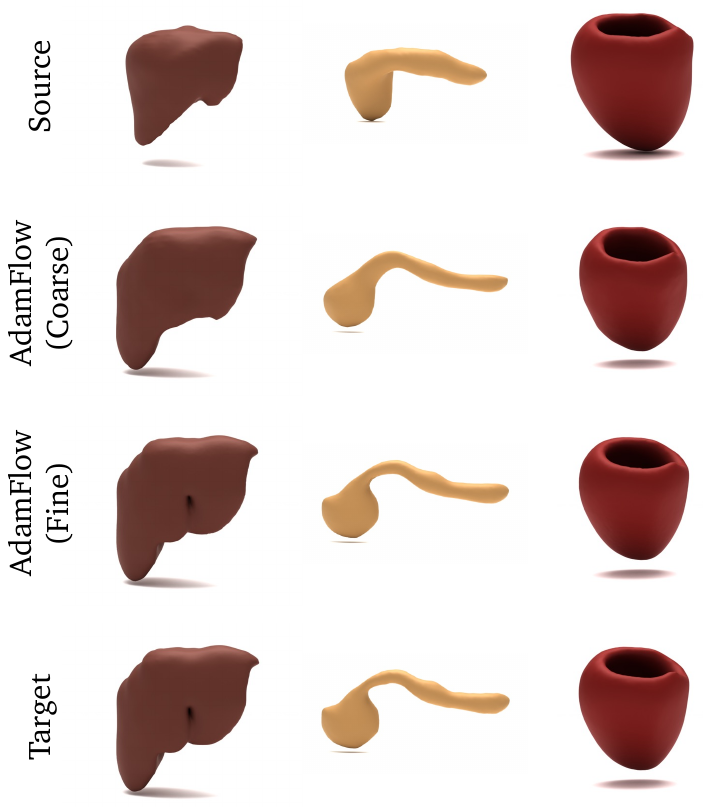}
\caption{Qualitative comparisons for non-rigid surface registration of the liver, pancreas, and left ventricle from a source mesh to a target mesh. AdamFlow is used to minimises the SWD for global surface alignment (coarse), and then optimises the Chamfer distance for local mesh refinement (fine), enabling coarse-to-fine surface registration.}
\label{fig:mesh_nonrigid}
\end{figure}

\subsection{Computational efficiency}
In this section, we examine the computational efficiency of the SWD-based objective functional and the AdamFlow optimisation algorithm. First, we compare the computational costs of SWD using different numbers of Monte Carlo projections $L$ with the ICP objective and Chamfer distance. We randomly sample two sets of $N$ points from a standard Gaussian distribution and compute their distance using different discrepancy metrics. We repeat such a procedure $1,000$ times and measure the average runtime required for each calculation. Figure~\ref{fig:runtime} compares the average runtime from $N=5,000$ to $50,000$ points. All distance metrics show $\mathcal{O}(N\log N)$ computational complexity with increasing numbers of points. The SWD with $L=4$ projections can be computed efficiently in 1.7 milliseconds (ms) for up to 50k points, while the ICP objective needs 5 ms and the Chamfer distance requires more than 10 ms due to bidirectional point matching.

Furthermore, we report the runtime of affine and non-rigid surface registration for all anatomical structures and all variants of the WGF-based optimisation methods in Table~\ref{tab:runtime}. Compared to other variants of the WGF, the proposed AdamFlow method only introduces a marginal increase in runtime, requiring a few seconds for affine and non-rigid surface registration tasks, while achieving significant improvement in registration performance.

\begin{figure}
\centering
\includegraphics[width=0.9\linewidth]{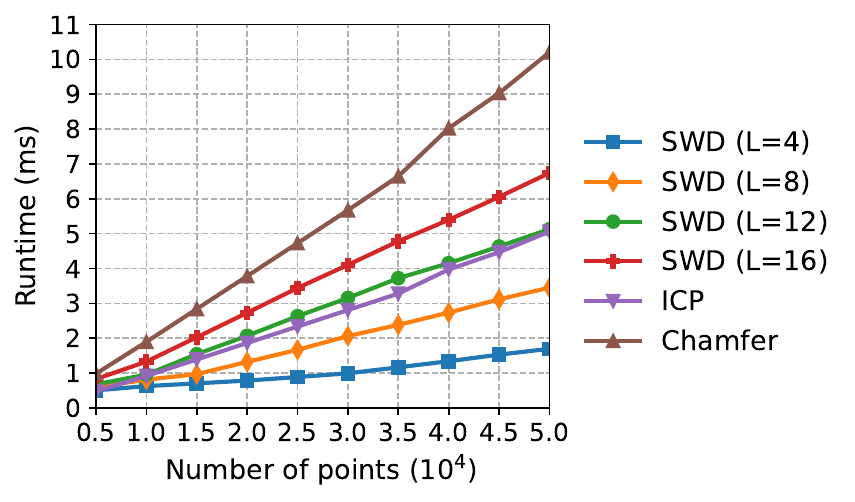}
\caption{Runtime (millisecond) for computing the discrepancy metrics across different numbers of points.}
\label{fig:runtime}
\end{figure}

\begin{table}
\setlength{\tabcolsep}{2.5pt}
\centering
\caption{Runtime (second) of different optimisation methods for affine and non-rigid surface registration of three organs.}
\begin{tabular}{ll|ccc}
\toprule
 & Method & Liver & Pancreas & Left ventricle\\
\midrule
\multirow{4}{*}{Affine}
& WGF & $2.47\pm0.33$ & $2.72\pm0.21$ & $1.54\pm0.06$ \\
& HBF & $2.53\pm0.34$ & $2.82\pm0.21$ & $1.60\pm0.05$ \\
& Nesterov & $2.53\pm0.34$ & $2.83\pm0.20$ & $1.60\pm0.06$ \\
& AdamFlow & $2.65\pm0.33$ & $3.07\pm0.21$ & $1.72\pm0.05$ \\
\midrule
\multirow{4}{*}{Non-rigid}
& WGF & $2.52\pm0.43$ & $1.40\pm0.15$ & $0.41\pm0.04$ \\
& HBF & $2.53\pm0.43$ & $1.43\pm0.15$ & $0.41\pm0.04$ \\
& Nesterov & $2.53\pm0.43$ & $1.43\pm0.15$ & $0.41\pm0.04$ \\
& AdamFlow & $2.55\pm0.43$ & $1.48\pm0.15$ & $0.42\pm0.04$ \\
\bottomrule
\end{tabular}
\label{tab:runtime}
\end{table}

\subsection{Hyperparameter selection}\label{sec:hyperparameter}
We conduct experiments on hyperparameter selection for objective functionals to explore how the hyperparameter tuning affects the performance of surface registration tasks. For AdamFlow method, the experiments for hyperparameter tuning are provided in Appendix \ref{sec:adamflow_param}. The registration accuracy is measured by ASSD and HD90 for all experiments.

\paragraph{Monte Carlo projection.} We investigate how the number of Monte Carlo projections $L$ in the SWD (\ref{eq:sliced_wasserstein_mc}) affects the surface registration performance. We consider different number of projections $L\in\{4,8,12,16\}$ in affine surface registration of pancreas using AdamFlow. Increasing the number of projections $L$ from $4$ to $8$ achieves a notable improvement in registration accuracy. Further increases of the projection number provide marginal improvements. However, as shown in Figure~\ref{fig:runtime}, the computational cost of the SWD increases in order of $\mathcal{O}(N\log N)$ with increasing number of projections.

\begin{figure}
\centering
\includegraphics[width=1.0\linewidth]{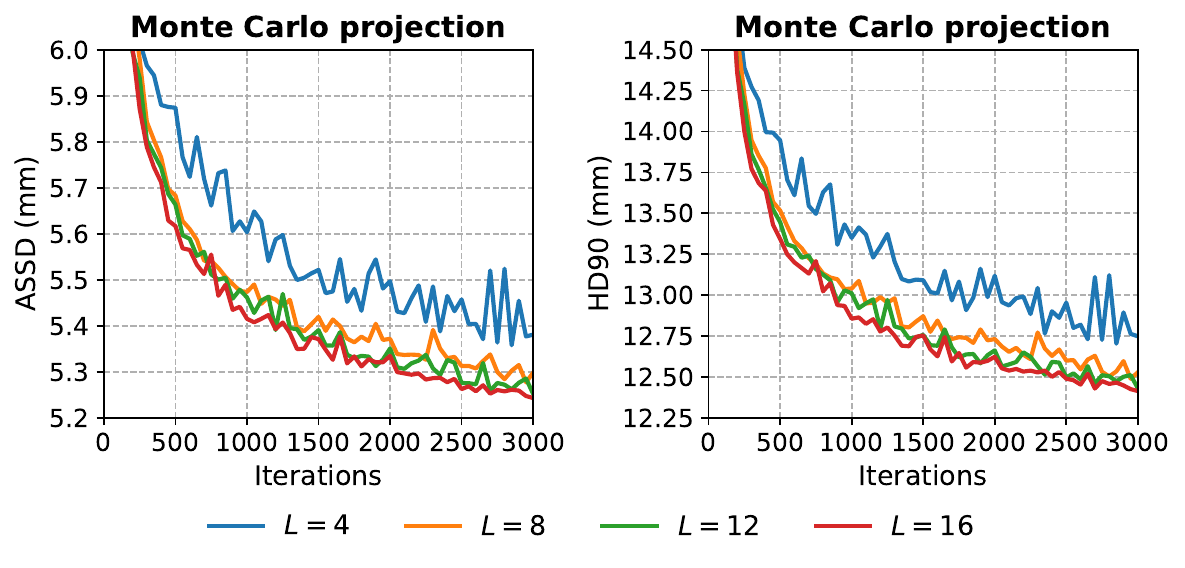}
\caption{Effect of different numbers of Monte Carlo projections $L$ on affine registration of pancreas surfaces with sliced Wasserstein discrepancy using AdamFlow.}
\label{fig:projection}
\end{figure}

\paragraph{Mesh Laplacian regularisation.} Lastly, we conduct experiments on different weights $\lambda_{\mathrm{lap}}$ of the mesh Laplacian regularisation term in non-rigid surface problems (\ref{eq:nonrigid_registration}). We consider non-rigid liver surface registration via AdamFlow with weights $\lambda_{\mathrm{lap}}\in\{0.0,2.0,4.0\}$. The results in Figure~\ref{fig:laplacian} show that the registration error increases substantially without the mesh Laplacian regularisation. Increasing the weight $\lambda_{\mathrm{lap}}$ from $2.0$ to $4.0$ slightly affects the registration accuracy. Moreover, a qualitative comparison across different regularisation weights $\lambda_{\mathrm{lap}}$ is illustrated in Figure~\ref{fig:mesh_laplacian}. The liver surface shows poor mesh quality without mesh regularisation.

\begin{figure}
\centering
\includegraphics[width=1.0\linewidth]{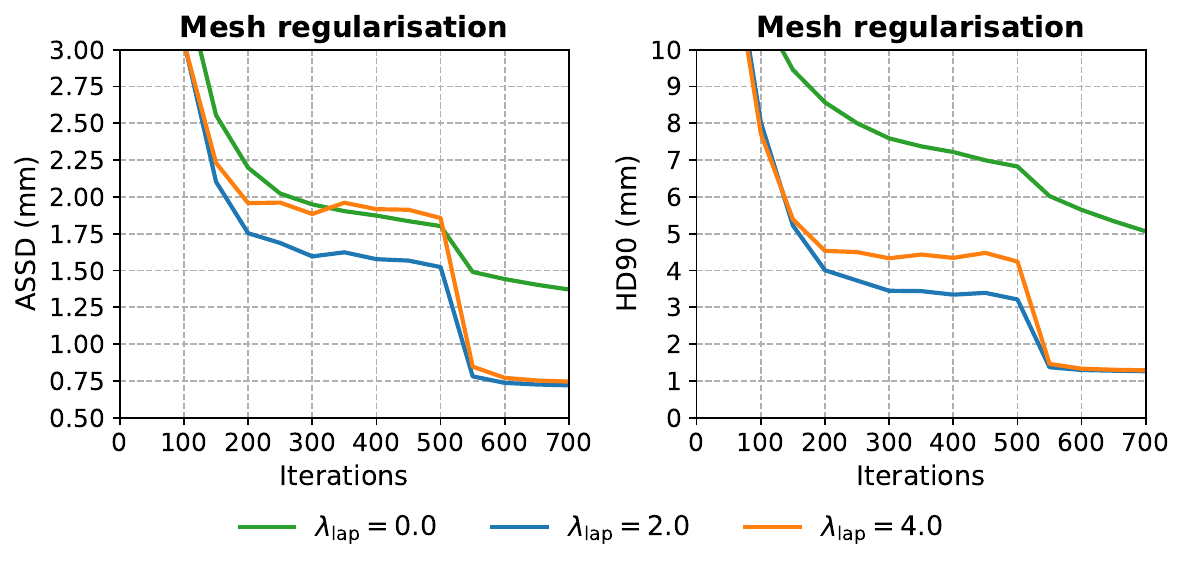}
\caption{Effect of mesh Laplacian regularisation weights $\lambda_{\mathrm{lap}}$ on non-rigid registration of liver surfaces via AdamFlow.}
\label{fig:laplacian}
\end{figure}

\begin{figure}
\centering
\includegraphics[width=1.0\linewidth]{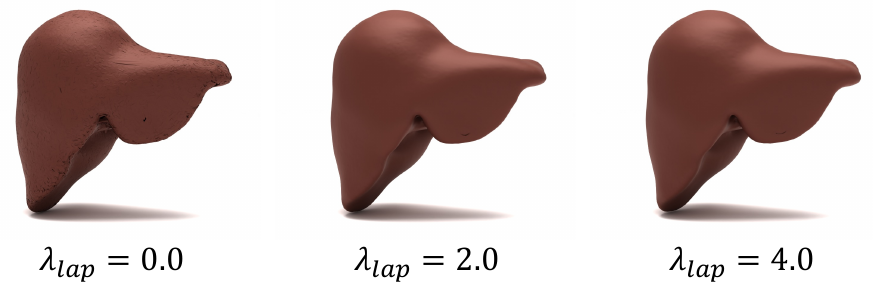}
\caption{Qualitative comparison for non-rigid registration of liver surfaces with different mesh Laplacian regularisation weights $\lambda_{\mathrm{lap}}$.}
\label{fig:mesh_laplacian}
\end{figure}

\section{Discussion}

In this work, we presented a distributional optimisation framework for accurate and efficient surface registration for medical imaging. For affine surface registration, the discrepancy between two surface meshes, which are represented as probability measures, is characterised by the SWD with $\mathcal{O}(N\log N)$ complexity. For non-rigid surface registration, we proposed a hybrid objective functional that incorporates the SWD and Chamfer distance for coarse-to-fine registration. To accelerate the distributional optimisation of the mesh alignment, we proposed AdamFlow, an adaptive WGF formulated by a PDE, extending the Adam optimisation from the Euclidean space to the space of probability measures. We theoretically proved that the solution of the AdamFlow PDE asymptotically converges to the set of critical points of the objective functional. Experiments on three different anatomical structures validated that the SWD-based metrics significantly improve the registration accuracy and computational efficiency than the local point matching approaches. In addition, AdamFlow solved non-convex surface registration problems effectively, achieving consistently better performance compared to momentum-based variants of the WGF across multiple anatomical structures.

One limitation of this work is that we only derived the asymptotic convergence property of AdamFlow without quantification of the convergence rate. This could be achieved by introducing additional conditions such as the convexity or Łojasiewicz property of the objective functional. Besides, we only regularised the mesh Laplacian smoothness for non-rigid surface registration to improve the mesh quality. There is no topological constraint that prevents the self-intersections of the surface meshes. Lastly, we only employed vanilla SWD to measure the discrepancy between the coordinates of the vertices or surface points, while the mesh properties such as the normals and curvatures were not considered. In future work, diffeomorphic registration frameworks could be incorporated to enforce topological consistency. We will also develop novel variants of the SWD to further improve the surface registration performance leveraging normal and curvature information.

\bibliographystyle{cas-model2-names}

\bibliography{ref}

\clearpage
\onecolumn
\appendix
\section{Theoretical and convergence analysis}

\subsection{Derivation of Adam ODE}
Following previous work \citep{barakat2021convergence}, the continuous Adam ODE system (\ref{eq:adam_ode}) is derived from the discrete Adam algorithm (\ref{eq:adam_discrete}). Note that the discrete Adam algorithm updates the biased estimation of the first and second moments in an exponential moving average (EMA) manner:
\begin{equation}\label{eq:discrete_ema}
m_{k+1}=\alpha m_{k} + (1-\alpha)\nabla f(x_{k}),~~~~
v_{k+1}=\beta v_{k} + (1-\beta)(\nabla f(x_{k}))^2.
\end{equation}
For Adam ODE, we substitute the discrete EMA with a continuous formula:
\begin{equation}\label{eq:continuous_ema}
\dot{m} = (1-\alpha)(\nabla f(x)-m),~~~~
\dot{v} = (1-\beta)((\nabla f(x))^2-v).
\end{equation}
In discrete Adam, the optimisation variable $x_k$ is updated according to the unbiased estimation of the first and second moments:
\begin{equation}
\hat{m}_k=m_k/(1-\alpha^k),~~~~\hat{v}_k=v_k/(1-\beta^k),~~~~x_{k} = x_{k-1}-\eta\frac{\hat{m}_k}{\sqrt{\hat{v}_k}+\epsilon},
\end{equation}
where $1-\alpha^k$ and $1-\beta^k$ are the coefficients for bias correction. Next, we derive the bias correction coefficients for the continuous EMA. For initial values $m(0)=v(0)=0$, the continuous EMA (\ref{eq:continuous_ema}) can be solved by
\begin{equation}\label{eq:ema_mv}
m(t)=(1-\alpha)\int_0^te^{(1-\alpha)(\tau-t)}\nabla f(x(\tau))\mathrm{d}\tau,~~~~
v(t)=(1-\beta)\int_0^te^{(1-\beta)(\tau-t)}(\nabla f(x(\tau)))^2\mathrm{d}\tau.
\end{equation}
Then, the bias correction coefficients can be computed by
\begin{equation}\label{eq:ema_bias}
(1-\alpha)\int_0^te^{(1-\alpha)(\tau-t)}\mathrm{d}\tau= 1-e^{-(1-\alpha)t},~~~~
(1-\beta)\int_0^te^{(1-\beta)(\tau-t)}\mathrm{d}\tau= 1-e^{-(1-\beta)t}.
\end{equation}
Therefore, the continuous optimisation variable $x(t)$ is evolved according to the following dynamics:
\begin{equation}
\dot{x} = -\eta g_t(m,v) = 
-\eta\frac{m/(1-e^{-(1-\alpha)t})}{\sqrt{v/(1-e^{-(1-\beta)t})}+\epsilon}.
\end{equation}
This completes the derivation of Adam ODE system (\ref{eq:adam_ode}).

\subsection{Derivation of AdamFlow}
In this section, we provide detailed proofs for Propositions \ref{prop1} and \ref{prop2} to demonstrate the equivalence between Adam ODE (\ref{eq:adam_ode}) and AdamFlow (\ref{eq:adamflow}) when the objective functional $F[\mu]$ is defined as the potential energy of the objective function $f$.
\vspace{6pt}

\noindent\textbf{Proposition 1.} If $(x(t),m(t),v(t))$ is a solution of Adam ODE (\ref{eq:adam_ode}) with an objective function $f$, then the trajectory of the Dirac measure $\delta_{(x(t),m(t),v(t))}$ is a solution of AdamFlow (\ref{eq:adamflow}) with the objective functional $F_f[\mu]=\int_{\mathbb{R}^d}f(x)\mathrm{d}\mu(x)$ in the distributional sense.

\begin{proof} We follow the proof procedure in previous work \citep{chen2025accelerating}. Suppose $(x(t),m(t),v(t))$ is a solution of the Adam ODE system (\ref{eq:adam_ode}), we need to verify the curve $\delta_{(x(t),m(t),v(t))}$ is a solution of AdamFlow (\ref{eq:adamflow}) in the distributional sense. Specifically, for any smooth test function $\varphi\in C_c^{\infty}(\mathbb{R}^d\times\mathbb{R}^d\times\mathbb{R}^d)$ with compact support, we need to verify that the curve $\delta_{(x(t),m(t),v(t))}$ satisfies the following weak formulation of the PDE (\ref{eq:adamflow}):
\begin{equation}\label{eq:adamflow_weak}
\begin{split} 
\frac{\mathrm{d}}{\mathrm{d}t}\int_{\mathbb{R}^{3d}}\rho_t\varphi
~\mathrm{d}x\mathrm{d}m\mathrm{d}v=
&-\int_{\mathbb{R}^{3d}}\nabla_m\cdot\left(\rho_t(1-\alpha)\left(\nabla_W F_f[\mu_t]-m\right)\right)\varphi~\mathrm{d}x\mathrm{d}m\mathrm{d}v\\
&-\int_{\mathbb{R}^{3d}}\nabla_v\cdot\left(\rho_t(1-\beta)\left((\nabla_W F_f[\mu_t])^2-v\right)\right)\varphi~\mathrm{d}x\mathrm{d}m\mathrm{d}v\\
&+\int_{\mathbb{R}^{3d}}\nabla_x\cdot\left(\rho_t\eta g_t(m,v)\right)\varphi~\mathrm{d}x\mathrm{d}m\mathrm{d}v.
\end{split}
\end{equation}
Let $\rho_t=\delta_{(x(t),m(t),v(t))}$ in Eq.~(\ref{eq:adamflow_weak}). Since $(x(t),m(t),v(t))$ is a solution of the Adam ODE (\ref{eq:adam_ode}), for the left-hand side, we have
\begin{equation}\label{eq:weak_lhs}
\begin{split}
\frac{\mathrm{d}}{\mathrm{d}t}\int_{\mathbb{R}^{3d}}\rho_t\varphi
~\mathrm{d}x\mathrm{d}m\mathrm{d}v
&=\frac{\mathrm{d}}{\mathrm{d}t}\int_{\mathbb{R}^{3d}}\delta_{(x(t),m(t),v(t))}\varphi~\mathrm{d}x\mathrm{d}m\mathrm{d}v
=\frac{\mathrm{d}}{\mathrm{d}t}\varphi(x(t),m(t),v(t))
=\nabla_x\varphi\cdot\dot{x}(t) + \nabla_m\varphi\cdot\dot{m}(t) + \nabla_v\varphi\cdot\dot{v}(t)\\
&=-\eta\nabla_x\varphi\cdot g_t(m,v)+(1-\alpha)\nabla_m\varphi\cdot(\nabla f(x(t))-m(t)) + (1-\beta)\nabla_v\varphi\cdot\left((\nabla f(x(t)))^2-v(t)\right).
\end{split}
\end{equation}
For the right-hand side, since $\nabla_W F_f[\mu]=\nabla\delta\int_{\mathbb{R}^d} f(x)\mathrm{d}\mu(x)=\nabla f$, we have 
\begin{equation}\label{eq:weak_rhs_1}
\begin{split}
-\int_{\mathbb{R}^{3d}}\nabla_m\cdot\left(\rho_t(1-\alpha)\left(\nabla_W F_f[\mu_t]-m\right)\right)\varphi~\mathrm{d}x\mathrm{d}m\mathrm{d}v
&=\int_{\mathbb{R}^{3d}}\rho_t(1-\alpha)\left(\nabla_W F_f[\mu_t]-m\right)\cdot\nabla_m\varphi~\mathrm{d}x\mathrm{d}m\mathrm{d}v\\
&=\int_{\mathbb{R}^{3d}}\delta_{(x(t),m(t),v(t))}(1-\alpha)\left(\nabla f-m\right)\cdot\nabla_m\varphi~\mathrm{d}x\mathrm{d}m\mathrm{d}v\\
&=(1-\alpha)\nabla_m\varphi\cdot(\nabla f(x(t))-m(t)).
\end{split}
\end{equation}
As $\varphi$ has compact support, the first equality holds according to the integration by parts. Similarly, we obtain
\begin{equation}\label{eq:weak_rhs_2}
-\int_{\mathbb{R}^{3d}}\nabla_v\cdot\left(\rho_t(1-\beta)\left((\nabla_W F_f[\mu_t])^2-v\right)\right)\varphi~\mathrm{d}x\mathrm{d}m\mathrm{d}v
=(1-\beta)\nabla_v\varphi\cdot((\nabla f(x(t)))^2-v(t)),
\end{equation}
and
\begin{equation}\label{eq:weak_rhs_3}
\int_{\mathbb{R}^{3d}}\nabla_x\cdot\left(\rho_t\eta g_t(m,v)\right)\varphi~\mathrm{d}x\mathrm{d}m\mathrm{d}v=-\eta\nabla_x\varphi\cdot g_t(m,v).
\end{equation}
Combining Eq.~(\ref{eq:weak_lhs}-\ref{eq:weak_rhs_3}), we have shown that the weak formulation (\ref{eq:adamflow_weak}) holds for any test function $\varphi$, and thus $\delta_{(x(t),m(t),v(t))}$ is a solution of AdamFlow (10) in the distributional sense.
\end{proof}

\noindent\textbf{Proposition 2.}
If $x(t)$ converges to a critical point of an objective function $f$, then the trajectory of the Dirac measure $\delta_{x(t)}$ converges to a critical point of the objective functional $F_f[\mu]=\int_{\mathbb{R}^d} f(x)\mathrm{d}\mu(x)$. 

\begin{proof}
Supposing $x(t)$ converges to a critical point $x_*$ of $f$ such that $\nabla f(x_*)=0$, then the curve $\delta_{x(t)}$ converges to a Dirac measure $\delta_{x_*}$, which satisfies 
\begin{equation}\label{eq:critical_point}
\int_{\mathbb{R}^{d}}\|\nabla_W F_f[\delta_{x_*}](x)\|~\mathrm{d}\delta_{x_*}(x)=\int_{\mathbb{R}^{d}}\|\nabla f(x)\|~\mathrm{d}\delta_{x_*}(x)=\|\nabla f(x_*)\|=0.
\end{equation}
Hence, $\delta_{x_*}$ is a critical point of the objective functional $F_f[\mu]$.
\end{proof}

\subsection{Convergence analysis}\label{appendix_A3}

\noindent\textbf{Theorem 1.} Assume the objective functional is lower bounded by $F[\mu]\geq\underline{F}$. Let $\rho_t\in\mathcal{P}_2(\mathbb{R}^d\times\mathbb{R}^d\times\mathbb{R}^d)$ be a solution of AdamFlow (\ref{eq:adamflow}) with $x$-marginal $\mu_t$ and initial condition $\rho_0(x,m,v)=\mu_0(x)\delta_{(0,0)}(m,v)$. If $\rho_t$ is absolute continuous and the coefficients satisfy $4\alpha-\beta<3$, then every limit point of the curve $\mu_t$ is a critical point of the objective functional $F[\mu]$. More specifically, let $\omega(\mu_0)$ denote an $\omega$-limit set that contains all limit points of $\mu_t$:
\begin{equation*}
\omega(\mu_0):=\big\{\mu_{\infty}~|~\mu_{\infty}=\rho_{\infty}^X,~\exists t_n\rightarrow\infty\mbox{ such that }\rho_{t_n}\rightarrow\rho_{\infty}\big\}.
\end{equation*}
Then $\omega(\mu_0)$ is a subset of the set of critical points of $F[\mu]$, i.e., $\omega(\mu_0)\subseteq\left\{\mu_*|\nabla_W F[\mu_*](x)=0,~\mbox{a.e.}~ x\in\mathrm{supp}(\mu_*)\right\}$.

\begin{proof}
We use Lyapunov approach \citep{barakat2021convergence,chen2025accelerating} to prove the convergence of AdamFlow (10). We define a Lyapunov functional:
\begin{equation}\label{eq:lyapunov}
\mathcal{E}_t[\rho_t]:=F[\mu_t]
+\frac{\eta}{2(1-\alpha)}\int_{\mathbb{R}^{3d}}\left\langle m,g_t(m,v)\right\rangle~\mathrm{d}\rho_t(x,m,v),
\end{equation}
where $\mu_t=\rho_t^X$ is the $x$-marginal of $\rho_t$. Since the objective functional $F[\mu_t]\geq\underline{F}$, the Lyapunov functional is also lower bounded, \emph{i.e.}, $\mathcal{E}_t[\rho_t]\geq\underline{F}$. Taking derivative with respect to time $t$, we obtain
\begin{equation}\label{eq:lyapunov_dt}
\begin{split}
\frac{d}{dt}\mathcal{E}_t[\rho_t]
=\int\delta F[\mu_t]\partial_t\rho_t
+\frac{\eta}{2(1-\alpha)}\int_{\mathbb{R}^{3d}}\partial_t\left\langle m, g_t(m,v)\right\rangle~\mathrm{d}\rho_t(x,m,v)+\frac{\eta}{2(1-\alpha)}\int\left\langle m, g_t(m,v)\right\rangle\partial_t\rho_t.
\end{split}
\end{equation}
For the first term, using integration by parts and by the fact that $\nabla_m\delta F[\mu_t]=\nabla_v\delta F[\mu_t]=0$, we have
\begin{equation}\label{eq:lyapunov_dt_1}
\begin{split}
\int\delta F[\mu_t]\partial_t\rho_t
&=\int\delta F[\mu_t]\left(\nabla_x\cdot\left(\rho_t\eta g_t(m,v)\right)-\nabla_m\cdot\left(\rho_t(1-\alpha)\left(\nabla_W F[\mu_t]-m\right)\right)-\nabla_v\cdot\left(\rho_t(1-\beta)\left((\nabla_W F[\mu_t])^2-v\right)\right)\right)\\
&=\int_{\mathbb{R}^{3d}}-\eta\left\langle\nabla_W F[\mu_t],g_t(m,v)\right\rangle~\mathrm{d}\rho_t(x,m,v).
\end{split}
\end{equation}
For the second term, we have
\begin{equation}\label{eq:lyapunov_dt_2}
\begin{split}
&\frac{\eta}{2(1-\alpha)}\int_{\mathbb{R}^{3d}}\partial_t\left\langle m, g_t(m,v)\right\rangle~\mathrm{d}\rho_t(x,m,v)
=\frac{\eta}{2(1-\alpha)}\int_{\mathbb{R}^{3d}}\partial_t\sum_{i=1}^d\frac{m_i^2/(1-e^{-(1-\alpha)t})}{\sqrt{v_i/(1-e^{-(1-\beta)t})}+\epsilon}~\mathrm{d}\rho_t(x,m,v)\\
=&\int_{\mathbb{R}^{3d}}\sum_{i=1}^d\frac{m_i^2/(1-e^{-(1-\alpha)t})}{\sqrt{v_i/(1-e^{-(1-\beta)t})}+\epsilon}
\left(-\frac{\eta}{2}\cdot\frac{e^{-(1-\alpha)t}}{1-e^{-(1-\alpha)t}}
+\frac{\eta(1-\beta)}{4(1-\alpha)}
\frac{e^{-(1-\beta)t}}{1-e^{-(1-\beta)t}}
\frac{\sqrt{v_i/(1-e^{-(1-\beta)t})}}{\sqrt{v_i/(1-e^{-(1-\beta)t})}+\epsilon}\right)~\mathrm{d}\rho_t(x,m,v)\\
\leq&\int_{\mathbb{R}^{3d}}
\left(-\frac{\eta}{2}\cdot\frac{1}{e^{(1-\alpha)t}-1}
+\frac{\eta(1-\beta)}{4(1-\alpha)}
\cdot\frac{1}{e^{(1-\beta)t}-1}\right)\langle m,g_t(m,v)\rangle~\mathrm{d}\rho_t(x,m,v).
\end{split}
\end{equation}
For the third term, similar to Eq.~(\ref{eq:lyapunov_dt_1}), we use integration by parts again and obtain
\begin{equation}\label{eq:lyapunov_dt_3}
\begin{split}
&\frac{\eta}{2(1-\alpha)}\int\left\langle m, g_t(m,v)\right\rangle\partial_t\rho_t\\
=&\frac{\eta}{2(1-\alpha)}\int_{\mathbb{R}^{3d}}(1-\alpha)\nabla_m\left\langle m, g_t(m,v)\right\rangle\cdot\left(\nabla_W F[\mu_t]-m\right)+(1-\beta)\nabla_v\left\langle m, g_t(m,v)\right\rangle\cdot\left((\nabla_W F[\mu_t])^2-v\right)~\mathrm{d}\rho_t(x,m,v)\\
=&\int_{\mathbb{R}^{3d}}\eta\left\langle\nabla_W F[\mu_t],g_t(m,v)\right\rangle
-\eta\left\langle m, g_t(m,v)\right\rangle\\
&~~~~~~~~+\frac{\eta(1-\beta)}{4(1-\alpha)}\sum_{i=1}^d\frac{m_i^2/(1-e^{-(1-\alpha)t})}{\left(\sqrt{v_i/(1-e^{-(1-\beta)t})}+\epsilon\right)^2}
\left(-\frac{(\nabla_W F[\mu_t])^2_i(x)}{\sqrt{v_i(1-e^{-(1-\beta)t})}}
+\sqrt{v_i/(1-e^{-(1-\beta)t})}\right)~\mathrm{d}\rho_t(x,m,v)\\
\leq & \int_{\mathbb{R}^{3d}}\eta\left\langle\nabla_W F[\mu_t],g_t(m,v)\right\rangle
-\left(\eta-\frac{\eta(1-\beta)}{4(1-\alpha)}\right)\left\langle m, g_t(m,v)\right\rangle~\mathrm{d}\rho_t(x,m,v).
\end{split}
\end{equation}
Combining Eq.~(\ref{eq:lyapunov_dt_1}-\ref{eq:lyapunov_dt_3}), it follows that
\begin{equation}\label{eq:lyapunov_dt_all}
\begin{split}
\frac{d}{dt}\mathcal{E}_t[\rho_t]\leq\int_{\mathbb{R}^{3d}}
-\eta\left(1-\frac{1-\beta}{4(1-\alpha)}+\frac{1}{2}\cdot\frac{1}{e^{(1-\alpha)t}-1}
-\frac{1-\beta}{4(1-\alpha)}
\cdot\frac{1}{e^{(1-\beta)t}-1}\right)\langle m,g_t(m,v)\rangle~\mathrm{d}\rho_t(x,m,v).
\end{split}
\end{equation}
Since $4\alpha-\beta<3$ and $\langle m,g_t(m,v)\rangle\geq0$, there exist a time $T>0$ and a constant $0<C<1-\frac{1-\beta}{4(1-\alpha)}$ such that for any $t\geq T$,
\begin{equation}\label{eq:lyapunov_dt_bound}
\frac{d}{dt}\mathcal{E}_t[\rho_t]\leq\int_{\mathbb{R}^{3d}}
-\eta C\langle m,g_t(m,v)\rangle~\mathrm{d}\rho_t(x,m,v)\leq 0.
\end{equation}
For $t\geq T$, the Lyapunov functional $\mathcal{E}_t[\rho_t]$ is monotonically decreasing and bounded $\underline{F}\leq\mathcal{E}_t[\rho_t]\leq\mathcal{E}_T[\rho_T]$. Hence, $\mathcal{E}_t[\rho_t]$ converges as $t\rightarrow\infty$, and we denote its limit by $\mathcal{E}_{\infty}:=\lim_{t\rightarrow\infty}\mathcal{E}_t[\rho_t]$. Integrating Eq.~(\ref{eq:lyapunov_dt_bound}) over $[T,\infty)$, we have 
\begin{equation}\label{eq:lyapunov_dt_integral}
0\leq\int_T^{\infty}\int_{\mathbb{R}^{3d}}
\langle m,g_t(m,v)\rangle~\mathrm{d}\rho_t(x,m,v)~\mathrm{d}t\leq \frac{1}{\eta C}\left(\mathcal{E}_T[\rho_T]-\mathcal{E}_{\infty}\right)<\infty.
\end{equation}
Based on the absolute continuity of $\rho_t$, we obtain that $\lim_{t\rightarrow\infty}\int_{\mathbb{R}^{3d}}
\langle m,g_t(m,v)\rangle~\mathrm{d}\rho_t(x,m,v)=0$ and thus $\mathcal{E}_{\infty}=\lim_{t\rightarrow\infty}F[\rho_t]$. According to the definition (\ref{eq:gt}) of $g_t(m,v)$ and the tightness of $\rho_t\in\mathcal{P}_2(\mathbb{R}^d\times\mathbb{R}^d\times\mathbb{R}^d)$, it can be further derived that $\lim_{t\rightarrow\infty}\int_{\mathbb{R}^{3d}}
\|m\|^2~\mathrm{d}\rho_t(x,m,v)=0$.

The tightness of $\rho_t$ ensures there exists a subsequence $t_n\rightarrow\infty$ such that $\rho_{t_n}\rightarrow\rho_{\infty}$ as $n\rightarrow\infty$. For any solution $\rho_t$ of AdamFlow (\ref{eq:adamflow}) with initial condition $\rho_0$, we show that the $\omega$-limit set $\omega(\mu_0)$ is contained in the set of critical points of the objective functional $F[\mu]$, \emph{i.e.}, $\omega(\mu_0)\subseteq\left\{\mu_*|\nabla_W F[\mu_*](x)=0,~\mbox{a.e.}~ x\in\mathrm{supp}(\mu_*)\right\}$.
For any test function $\varphi(m)\in C_c^{\infty}(\mathbb{R}^d)$, we test it on the AdamFlow (\ref{eq:adamflow}):
\begin{equation}\label{eq:adamflow_test_m}
\begin{split} 
\frac{\mathrm{d}}{\mathrm{d}t}\int_{\mathbb{R}^{3d}}\rho_t
\varphi(m)~\mathrm{d}x\mathrm{d}m\mathrm{d}v=
&-\int_{\mathbb{R}^{3d}}\nabla_m\cdot\left(\rho_t(1-\alpha)\left(\nabla_W F[\mu_t]-m\right)\right)\varphi(m)~\mathrm{d}x\mathrm{d}m\mathrm{d}v\\
&-\int_{\mathbb{R}^{3d}}\nabla_v\cdot\left(\rho_t(1-\beta)\left((\nabla_W F[\mu_t])^2-v\right)\right)\varphi(m)~\mathrm{d}x\mathrm{d}m\mathrm{d}v
+\int_{\mathbb{R}^{3d}}\nabla_x\cdot\left(\rho_t\eta g_t(m,v)\right)\varphi(m)~\mathrm{d}x\mathrm{d}m\mathrm{d}v\\
=&-\int_{\mathbb{R}^{3d}}\rho_t(1-\alpha)m\cdot\nabla\varphi(m)~\mathrm{d}x\mathrm{d}m\mathrm{d}v
+\int_{\mathbb{R}^{3d}}\rho_t(1-\alpha)\nabla_W F[\mu_t]\cdot\nabla\varphi(m)~\mathrm{d}x\mathrm{d}m\mathrm{d}v.
\end{split}
\end{equation}
Taking the subsequence $t_n$, we obtain 
\begin{equation}\label{eq:adamflow_converge_F}
\int_{\mathbb{R}^{3d}}\rho_{t_n} \nabla_W F[\mu_{t_n}]\cdot\nabla\varphi(m)~\mathrm{d}x\mathrm{d}m\mathrm{d}v=
\int_{\mathbb{R}^{3d}}\rho_{t_n} m\cdot\nabla\varphi(m)~\mathrm{d}x\mathrm{d}m\mathrm{d}v+\frac{1}{1-\alpha}\frac{\mathrm{d}}{\mathrm{d}t}\int_{\mathbb{R}^{3d}}\rho_t\varphi(m)~\mathrm{d}x\mathrm{d}m\mathrm{d}v\bigg|_{t=t_n}.
\end{equation}
For the first term on the right-hand side of Eq.~(\ref{eq:adamflow_converge_F}), it holds
\begin{equation}\label{eq:adamflow_converge_F_1}
\left|\int_{\mathbb{R}^{3d}}\rho_{t_n} m\cdot\nabla\varphi(m)~\mathrm{d}x\mathrm{d}m\mathrm{d}v\right|
\leq\int_{\mathbb{R}^{3d}}\rho_{t_n}\|m\|\cdot\|\nabla\varphi(m)\|~\mathrm{d}x\mathrm{d}m\mathrm{d}v
\leq\|\nabla\varphi(m)\|_{\infty}\left(\int_{\mathbb{R}^{3d}}\rho_{t_n}\|m\|^2~\mathrm{d}x\mathrm{d}m\mathrm{d}v\right)^{\frac{1}{2}}\rightarrow0,
\end{equation}
for $n\rightarrow\infty$, where $\|\nabla\varphi(m)\|_{\infty}<\infty$ since $\varphi\in C_c^{\infty}(\mathbb{R}^d)$ is smooth and has compact support. For the second term, since $\rho_t\in\mathcal{P}_2(\mathbb{R}^d\times\mathbb{R}^d\times\mathbb{R}^d)$ and $\varphi\in C_c^{\infty}(\mathbb{R}^d)$, we have
\begin{equation}\label{eq:adamflow_converge_F_2_integral}
\left|\lim_{n\rightarrow\infty}\int_{0}^{t_n}\frac{\mathrm{d}}{\mathrm{d}t}\int_{\mathbb{R}^{3d}}\rho_t\varphi(m)~\mathrm{d}x\mathrm{d}m\mathrm{d}v\right|
=\left|\int_{\mathbb{R}^{3d}}\rho_{\infty}\varphi(m)~\mathrm{d}x\mathrm{d}m\mathrm{d}v-\int_{\mathbb{R}^{3d}}\rho_0\varphi(m)~\mathrm{d}x\mathrm{d}m\mathrm{d}v\right|<\infty.
\end{equation}
Consequently, there exists a subsequence $t_{n_k}\rightarrow\infty$ of $t_n$ such that 
\begin{equation}\label{eq:adamflow_converge_F_2}
\lim_{k\rightarrow\infty}\frac{\mathrm{d}}{\mathrm{d}t}\int_{\mathbb{R}^{3d}}\rho_t\varphi(m)~\mathrm{d}x\mathrm{d}m\mathrm{d}v\bigg|_{t=t_{n_k}}\rightarrow0.
\end{equation}
Taking Eq.~(\ref{eq:adamflow_converge_F_1}) and (\ref{eq:adamflow_converge_F_2}) together,
\begin{equation}
\int_{\mathbb{R}^{3d}}\rho_{\infty} \nabla_W F[\mu_{\infty}]\cdot\nabla\varphi(m)~\mathrm{d}x\mathrm{d}m\mathrm{d}v=\lim_{k\rightarrow\infty}\int_{\mathbb{R}^{3d}}\rho_{t_{n_k}} \nabla_W F[\mu_{t_{n_k}}]\cdot\nabla\varphi(m)~\mathrm{d}x\mathrm{d}m\mathrm{d}v=0.
\end{equation}
Since $\varphi(m)$ is an arbitrary test function, we can conclude that $\nabla_W F[\mu_\infty](x)=0$ for a.e. $x\in\textrm{supp}(\mu_{\infty})$, which means $\omega(\mu_0)\subseteq\left\{\mu_*|\nabla_W F[\mu_*](x)=0,~\mbox{a.e.}~ x\in\mathrm{supp}(\mu_*)\right\}$. In other words, for any limit points $\rho_{\infty}$ of the solution $\rho_t$ of AdamFlow (\ref{eq:adamflow}), its $x$-marginal $\mu_{\infty}\in\omega(\mu_0)$ is a critical point of the objective functional $F[\mu]$.

\end{proof}

\subsection{Particle-based approximation}

\noindent\textbf{Proposition 3.} Suppose the solution $\rho_t$ of AdamFlow (\ref{eq:adamflow}) is approximated by the empirical measure $\hat{\rho}_t$ defined in (\ref{eq:empirical}). Then, the particle $(X^n,M^n,V^n)$ evolves according to the following ODE system:
\begin{equation*}
\left\{\begin{aligned}
\dot{M}^n & = (1-\alpha)\left(\nabla_W F[\hat{\mu}_t](X^n)-M^n\right), \\
\dot{V}^n & = (1-\beta)\left((\nabla_W F[\hat{\mu}_t](X^n))^2-V^n\right), \\
\dot{X}^n & = -\eta g_t(M^n,V^n),
\end{aligned} \right.
\end{equation*}
for $n=1,...,N$, where $(X^n(0),M^n(0),V^n(0))=(X_0^n,0,0)$ is the initial value of the particles.

\begin{proof}
Let $\hat{\mu}_t=\hat{\rho}_t^X:=\frac{1}{N}\sum_{n=1}^N\delta_{X^n(t)}$ denote the $x$-marginal of $\hat{\rho}_t$. For any test function $\varphi\in C_c^{\infty}(\mathbb{R}^d\times\mathbb{R}^d\times\mathbb{R}^d)$, the weak formulation of the AdamFlow PDE (10) is expressed as
\begin{equation}\label{eq:weak_particle}
\begin{split}
\int_{\mathbb{R}^{3d}}
\bigg[\partial_t\hat{\rho}_t
&-\nabla_x\cdot(\hat{\rho}_t\eta g_t(m,v))
+\nabla_m\cdot\left(\hat{\rho}_t(1-\alpha)\left(\nabla_W F[\hat{\mu}_t]-m\right)\right)+\nabla_v\cdot\left(\hat{\rho}_t(1-\beta)\left((\nabla_W F[\hat{\mu}_t])^2-v\right)\right)\bigg]\varphi~\mathrm{d}x\mathrm{d}m\mathrm{d}v \\
=\frac{1}{N}\sum_{n=1}^N\bigg[&\left(\nabla_x\varphi\cdot\dot{X}^n + \nabla_m\varphi\cdot\dot{M}^n + \nabla_v\varphi\cdot\dot{V}^n\right)
+\eta\nabla_x\varphi\cdot g_t(M^n,V^n)\\
&-\nabla_m\varphi\cdot(1-\alpha)(\nabla_W F[\hat{\mu}_t](X^n)-M^n)
-\nabla_v\varphi\cdot(1-\beta)\left((\nabla_W F[\hat{\mu}_t](X^n))^2-V^n\right)
\bigg]\\
=\frac{1}{N}\sum_{n=1}^N \bigg[&\nabla_x\varphi\cdot(\dot{X}^n+\eta g_t(M^n,V^n))
+\nabla_m\varphi\cdot\left(\dot{M}^n - (1-\alpha)(\nabla_W F[\hat{\mu}_t](X^n)-M^n))\right) \\
&+\nabla_v\varphi\cdot\left(\dot{V}^n-(1-\beta)\left((\nabla_W F[\hat{\mu}_t](X^n))^2-V^n\right)\right)\bigg]=0.
\end{split}
\end{equation}
Therefore, for $n=1,...,N$, we have
\begin{equation}\label{eq:system_particle}
\dot{X}^n+\eta g_t(M^n,V^n)=0,~~~\dot{M}^n - (1-\alpha)(\nabla_W F[\hat{\mu}_t](X^n)-M^n))=0,~~~~\dot{V}^n-(1-\beta)\left((\nabla_W F[\hat{\mu}_t](X^n))^2-V^n\right)=0,
\end{equation}
such that Eq.~(\ref{eq:weak_particle}) holds for arbitrary test function $\varphi$. This concludes the proof.

\end{proof}

\section{Hyperparameter tuning of AdamFlow}\label{sec:adamflow_param}

We conduct experiments for the hyperparameter tuning of AdamFlow, including the learning rate $\eta$ and exponential decay rates $\alpha$ and $\beta$. The registration accuracy is measured by ASSD and HD90 for all experiments.

\paragraph{Learning rate.} We adjust the learning rate $\eta$ of AdamFlow for the affine registration of the left ventricle surfaces. We let the learning rate $\eta\in\{0.001,0.005,0.01,0.05,0.1\}$ and the results are shown in Figure~\ref{fig:learning_rate}. Decreasing the learning rate $\eta$ from $0.01$ to $0.005$ has minimal effect on convergence, while a further decrease to $0.001$ results in slower convergence. Learning rates greater than $0.01$ cause unstable optimisation.

\paragraph{Exponential decay rate.} We tune the exponential decay rates $\alpha$ and $\beta$ of AdamFlow for the affine registration of the left ventricle surfaces. For combinations of $\alpha\in\{0.8,0.9\}$ and $\beta\in\{0.85,0.95\}$, the performance of AdamFlow is presented in Figure~\ref{fig:decay_rate}. Figure~\ref{fig:decay_rate} indicates that different combinations of decay rates lead to similar convergence behaviour, demonstrating that AdamFlow is not sensitive to decay rates.

\begin{figure}
\centering
\begin{subfigure}{0.51\textwidth}
\centering
\includegraphics[width=1.0\linewidth]{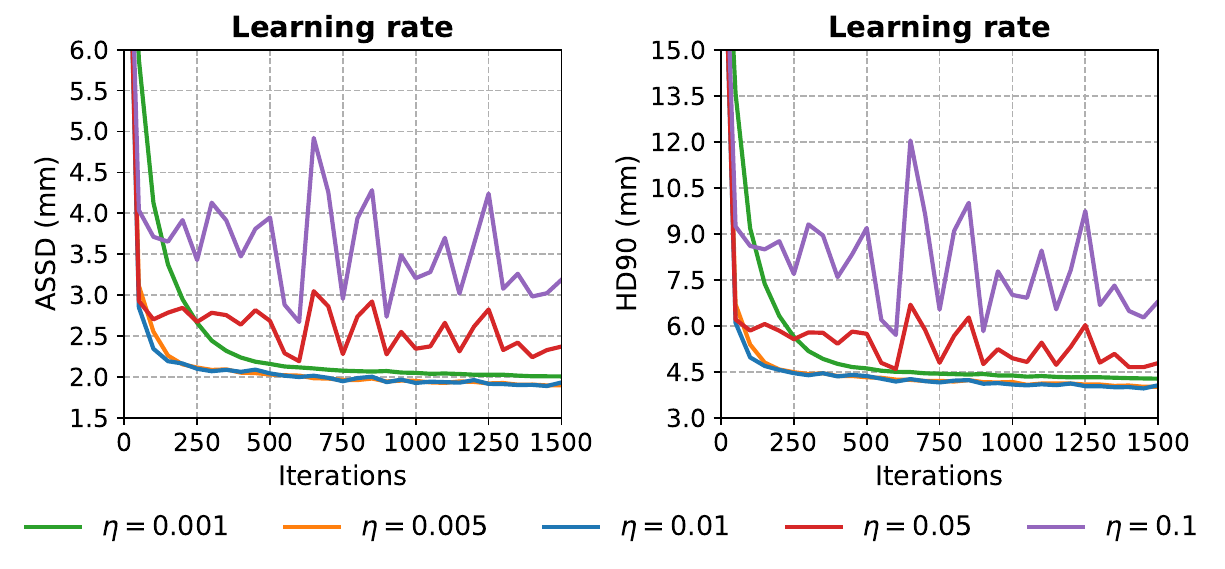}
\caption{Effect of different learning rates $\eta$.}
\label{fig:learning_rate}
\end{subfigure}
\hfill
\begin{subfigure}{0.47\textwidth}
\centering
\includegraphics[width=1.0\linewidth]{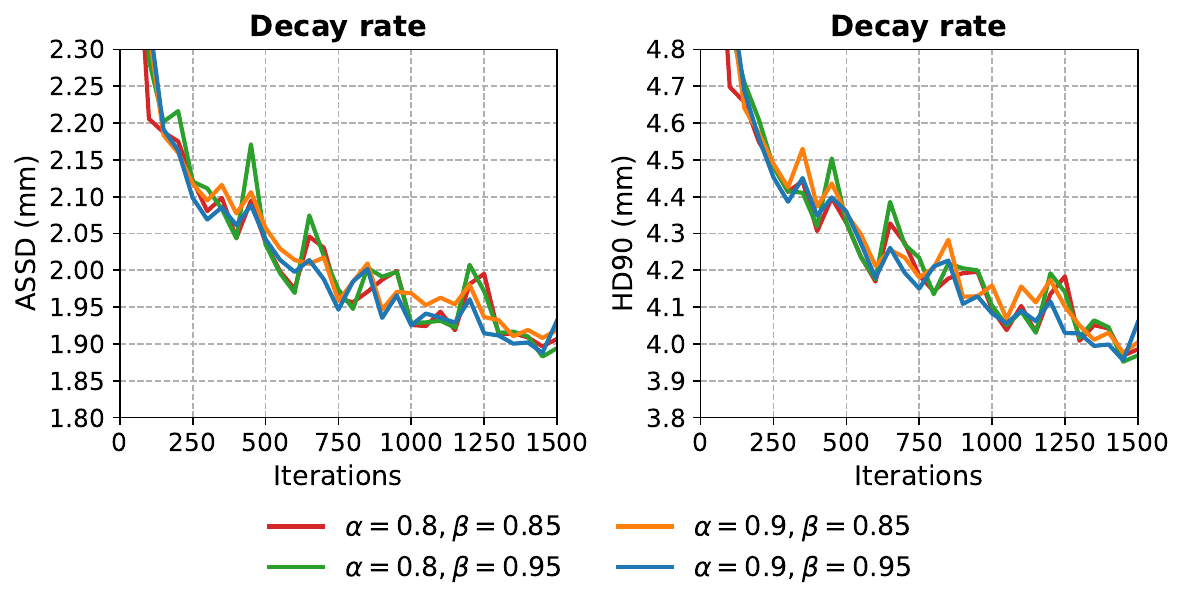}
\caption{Effect of different exponential decay rates $\alpha$ and $\beta$.}
\label{fig:decay_rate}
\end{subfigure}
\caption{Effect of hyperparameters on affine registration of left ventricle surfaces using AdamFlow.}
\end{figure}

\end{document}